\documentclass[twoside]{article}

\usepackage[accepted]{aistats2026}
\usepackage[utf8]{inputenc} 
\usepackage[T1]{fontenc}    
\usepackage{hyperref}       
\usepackage{url}            
\usepackage{booktabs}       
\usepackage{amsfonts}       
\usepackage{nicefrac}       
\usepackage{microtype}      
\usepackage{xcolor}         
\usepackage{graphicx}
\usepackage{aux/macros}
\usepackage{wrapfig}         
\usepackage[round]{natbib}

\begin{document}

%

%

\twocolumn[

\aistatstitle{Breaking Data Symmetry is Needed For Generalization in Feature Learning Kernels}

\aistatsauthor{ Marcel Tomàs \And Neil Mallinar \And Mikhail Belkin }

\aistatsaddress{ CFIS UPC \And  UC San Diego \And UC San Diego } ]

\begin{abstract}
    Grokking occurs when a model achieves high training accuracy but generalization to unseen test points happens long after that.
    This phenomenon was initially observed on a class of algebraic problems, such as learning modular arithmetic \citep{power2022grokking}.
    We study grokking on algebraic tasks in a class of feature learning kernels via the Recursive Feature Machine (RFM) algorithm \citep{rfm_science}, which iteratively updates feature matrices through the Average Gradient Outer Product (AGOP) of an estimator in order to learn task-relevant features.  Our main experimental finding is that generalization occurs only when a certain symmetry in the training set is broken. Furthermore, we empirically show that RFM generalizes by recovering the underlying symmetry group action inherent in the data. 
    We find that the learned feature matrices encode specific elements of the symmetry group, explaining the dependence of generalization on symmetry.\footnote{Code available at: \url{https://github.com/marceltomas/breaking-data-symmetries}.}
\end{abstract}

\section{INTRODUCTION}

The phenomenon of grokking has attracted much attention in recent literature \citep{power2022grokking,kumar2024grokking,mohamadi2023grokking,gromov2023grokking,liu2023omnigrokgrokkingalgorithmicdata,mallinar2025emergence}.
Grokking refers to the delayed generalization effect where a model perfectly fits (or \textit{interpolates}) the training dataset almost immediately, but performs no better than random guessing on the test data.
After continued training, sometimes for thousands of additional optimization steps \citep{power2022grokking}, the test performance improves from random guessing to nearly perfect performance, seemingly for no reason.
This is often considered to be an  example of so-called emergent phenomena in deep learning, and was initially shown in transformer models that are trained on algebraic tasks, such as imputing missing entries of a Cayley table for modular arithmetic \citep{power2022grokking}.

Recently, \cite{mallinar2025emergence} showed that such emergent curves in modular arithmetic are not unique to neural architectures nor gradient descent based optimization methods, and indeed appear in a class of feature learning kernels via the Recursive Feature Machine (RFM) \citep{rfm_science} algorithm.
RFM, when applied to a kernel, operates by iteratively fitting a Mahalanobis kernel predictor on training data, and then updating the feature matrix via the Average Gradient Outer Product (AGOP), which learns task-specific features.

\citet{mallinar2025emergence} show that RFM, with the Gaussian and quadratic kernels, learns block-circulant feature transformations (on data) to solve modular arithmetic tasks (addition, subtraction, multiplication, and division).
They argue that grokking is a manifestation of a gradual feature learning process, which can be understood via the Average Gradient Outer Product (AGOP), and that these feature transformations on the input data are sufficient to enable generalization. We build our study upon this line of work, finding the group-theoretic interpretation of such feature transformations by studying the symmetry group of the problem, and how the structure of the data has an impact on the behavior of RFM and the features it learns. 

More precisely, we examine how RFM learns features on a broader class of algebraic tasks, extending beyond modular arithmetic to other Abelian groups. We do so by turning to the group-theoretic properties of modular arithmetic, and more generally Abelian groups, and characterize more precisely the way in which AGOP in RFM learns features from group structure. Binary operations are a very studied setting in grokking \citep{mlpGrokking22,morwani2024featureemergencemarginmaximization, stander2024grokkinggroupmultiplicationcosets}, and we narrow our scope to the set of binary Abelian group addition due to the structure of its symmetry group.

We find that the target function in this setting can be entirely described by its symmetry group, and thus every data sample can be transformed into any other data sample with the same label through transformations described by such group. Understanding the symmetry group inherent in the data gives us a framework to describe the learned features, as well as design data partitions that trap RFM in specific symmetry-induced failure modes, inhibiting generalization. Our core contributions include:
\begin{enumerate}
    \item We experimentally find that RFM generalizes by learning the symmetry group of the target function, and breaking symmetry in the training data is needed to recover generalization.
    \item We propose a description of the data partitions in modular arithmetic and other Abelian groups that, as we show empirically, inhibit generalization in RFM.
    \item Empirically, we show that features learned by RFM align to representations of the symmetry group for the target function of the problem, and initializing RFM features that encode a subgroup of the symmetry group leads to generalization only within the orbit of that subgroup action.
\end{enumerate}

\section{PRELIMINARIES} \label{sec:prelim}
The task of modular arithmetic corresponds to the more general task of binary group operations. Indeed, the set of integers modulo $p$ under modular addition $\mathbb{Z}_p$ is isomorphic to the cyclic group of order $p$. Also, the multiplicative set of integers modulo $p$ (removing the 0) under modular multiplication $\mathbb{Z}^*_p$ is a cyclic group of order $p-1$ (if $p$ is prime). As such, it can be understood as filling in missing values of the Cayley table (see Section \ref{sec:groups}) for the cyclic group $C_p$ for modular addition, and $C_{p-1}$ for modular multiplication, for some prime number $p$. We further consider a generalization of this setting to other Abelian groups.

We follow the data setup of \citet{gromov2023grokking,mallinar2025emergence} and describe it here. Our goal is to learn a binary function $f(a,b)=a \star b = c$, where $a,b,c \in \mathcal{A}$ for an Abelian group $\mathcal{A}$ of order $n$ with operation $\star$. Group elements are individually one-hot encoded into vectors $e_i \in \mathbb{R}^n$. Input training data points, $x$, are formed by pairs of group elements $(a,b)$ which are concatenated to form data as $x = e_a \vert\vert e_b \in \mathbb{R}^{2n}$. The label for each point, $y \in \mathbb{R}^n$, is given by the one-hot encoded representation of the corresponding entry in the Cayley table, $e_{a \star b} \in \mathbb{R}^n$. Given a Cayley table we obtain $n^2$ total points for an Abelian group of order $n$. 

For addition and subtraction modulo $p$, we get $p^2$ points. On the other hand, we get ${(p-1)^2}$ total points for multiplication and division modulo $p$, since 0 is not a group element. We then split these points into a train set and test set where the proportion of points in the train set is referred to as the training fraction, $r \in [0, 1]$.

\subsection{Feature learning with Recursive Feature Machines (RFM)} \label{sec:rfm}

\cite{rfm_science} propose a class of feature learning kernels called Recursive Feature Machines (RFM), which learn features through iteratively updating features via the Average Gradient Outer Product (AGOP) of the kernel after fitting data.
We first define AGOP formally and then provide the RFM algorithm for completeness.\\ 

\begin{definition}[Average Gradient Outer Product (AGOP)]
    Given $n$ data points, $X \in \R^{n \times d}$, and a function with $c$ outputs given by $f: \R^d \to \R^c$, denote the Jacobian of $f$ with respect to a point $x_i \in X$ by $J_{f}(x_i) \in \R^{c \times d} = \frac{\partial f(x_i)}{\partial x}$. Then AGOP is defined as,
    \begin{align*}
        \agop(f, X) &= \frac{1}{n}\sum_{i=1}^n J_{f}(x_i)^T J_{f}(x_i) \in \R^{d \times d}.
    \end{align*}
\end{definition}
\cite{rfm_science} show that the AGOP of a neural network highlights task-relevant feature directions with respect to the input data and is highly correlated to the covariance matrix of the weights of the neural network throughout training.
From this observation, they propose RFM as an iterative process which enables kernels (that do not natively have a feature learning mechanism) to learn features.

Given training data $X \in \R^{n \times d}$, training labels $y \in \R^{n \times c}$, a Mahalanobis kernel $k_M(\cdot, \cdot): \R^d \times \R^d \to \R$, and an initial positive semi-definite matrix $M_0 \in \R^{d \times d}$ (often taken to be the identity, meaning we impose no prior structure on the features at the start of RFM). 
RFM applies the following two steps for $T$ iterations, starting at $t = 0$:
\begin{enumerate}
    \item Estimator fitting: construct kernel $K \in \R^{n \times n}$ where $K_{i,j} = k_{M_t}(x_i, x_j)$. Fit parameters $\alpha = K^{-1}y$ to obtain estimator $\hat{f}(\cdot) = \sum_{i=1}^n \alpha_i k_{M_t}(x_i, \cdot)$.
    \item Feature updates: $M_{t+1} = (\agop(\hat{f}, X))^s$ for some power $s > 0$ (commonly taken to be $s = 1/2$).
\end{enumerate}
In this work we consider the Mahalanobis quadratic kernel, $k_M(x, x') = (x^TMx')^2$, and the Mahalanobis Gaussian kernel, $k_M(x, x') = \exp(\frac{-\snorm{x - x'}^2_M}{L})$ where $\snorm{a}_M^2 = a^T M a$ and $L \in \R$ is a kernel bandwidth parameter. 
We choose these kernels and bandwidth parameter ($L =2.5$) from prior work which found they grok modular arithmetic \citep{mallinar2025emergence}.

\subsection{Group Theory} \label{sec:groups}

In this section, we provide a brief summary of group theory, covering the fundamental definitions that are relevant to this work.
\paragraph{Groups.}
A group $G$ is a nonempty set equipped with a binary operation $\star$ that combines two elements $a,b \in G$ to form another element of $G$, such that it satisfies associativity, has an identity element, and all elements have inverses under the operation. A group $G$ is \textit{Abelian} if the group operation is commutative, i.e., $a \star b = b \star a$ for all $a,b \in G$.
\paragraph{Subgroups. } Given a group $G$, a subgroup $H$, denoted $H \leq G$, is a subset of $G$ that satisfies the group axioms.

\paragraph{Cayley tables.} Let $G = \{g_1, g_2, \cdots g_p\}$ be a finite group of order $p$ equipped with a binary operation $\star$. The Cayley table associated with $(G, \star)$ characterizes the structure of the group under this operation and can be visualized as a standard ``multiplication table" with elements representing $g_i \star g_j$ for all pairs $i, j$. 


\paragraph{Homomorphisms. } If $(G,\star)$ and $(H,\star)$ are groups, a \textit{homomorphism} is a function $\varphi:G\to H$ such that $\varphi(x\star y)=\varphi(x)\star\varphi(y)$ for all $x,y\in G$. In other words, it respects the group operation. An \textit{isomorphism} is a \textit{bijective} homomorphism. An \textit{automorphism} is an isomorphism from a group to itself.

\paragraph{Group actions.}
A group action is a way in which a group $G$ can act on some set $X$ by transforming its elements while respecting the group structure. Formally, if $G$ is a group with identity element $e$, a group action is a function $\varphi:G \times X \to X$, often written as $\varphi(g,x)=g x$, that satisfies two key properties:
\begin{enumerate}
\item Identity element: for the identity element $e \in G$, we have that $e x = x$ for all $x \in X$.
\item Compatibility: for any $g,h \in G$ and $x \in X$, the action must satisfy the rule $g  (h x) = (g \star h) x$.
\end{enumerate}
Group actions provide a mathematical framework for describing and analyzing symmetries. 
A group action that preserves certain properties of interest, for example distances or angles, is referred to as a symmetry.

\paragraph{Orbit under a group. } Given a group $G$ acting on a set $X$, the orbit of an element $x \in X$ is the collection of all points reachable from $x$ under the group action:
$$
\text{Orb}_G(x) = \{g  x \mid g \in G\}.
$$
More generally, the orbit of a subset $X' \subseteq X$ is the union of the orbits of its elements, i.e.
$$
\text{Orb}_G(X') = \bigcup_{x \in X'} \text{Orb}_G(x) = \{g  x \mid g \in G, x \in X'\}.
$$
\paragraph{Symmetry groups.} Let $G$ be a group acting on a set $X$. Let $f:X \to Y$ be a function with $x \in X$ for some set $Y$. The function $f$ is  \textit{invariant} under the action of $G$ if $f(g x) = f(x)$ for all $x \in X$ and $g\in G$. The \textit{symmetry group} (also called \textit{invariance group}) of the function $f$ is defined as:
$$ \text{Sym}_f := \{g \in \text{Aut}(X) \mid f(g x)=f(x) \:\:\forall\:\: x \in X\}$$
where $\text{Aut}(X)$ is the set of all automorphisms on $X$, with composition as the group operation. If $f$ is a classifier, this corresponds to the set of transformations that do not change the output label.

\section{BREAKING DATA SYMMETRY ENABLES GENERALIZATION} \label{sec:breaking}

Using the setup in Section \ref{sec:prelim}, we train RFM with one-hot encoded input-label pairs on modular arithmetic (we generalize results on other Abelian groups in Appendix \ref{sec:fixed_point_experiments} and \ref{sec:generalization_experiments}).
Prior work shows that RFM groks addition, subtraction, multiplication and division \citep{mallinar2025emergence}.
We replicate those results with Gaussian and quadratic kernels on random 50\% training partitions for $p=61$ (out of $p^2$ total samples), which is sufficient for RFM to attain $100\%$ test accuracy. 
Figure \ref{fig:circulant_features} (first row) shows the feature matrices after training.

In contrast to random i.i.d. partitions, specific \textit{symmetric} partitions yield unexpected behaviors: (i) training on all samples except one (i.e. $\sim 99\%$ train fraction) prevents grokking and the model mispredicts the held-out point; (ii) when holding out two points, RFM generalizes \textit{only} if the pair of points breaks symmetry (i.e. both aren't fixed under a common reflection).
We formalize this notion of symmetry and experimentally demonstrate this below and provide extended results in Appendix \ref{sec:fixed_point_experiments}.\\


\begin{finding}
    When training RFM to learn modular arithmetic with symmetry group $D_{2p}$ for prime modulus $p$, RFM fails to generalize only when using train-test partitions that are invariant under the action of a non-singleton subgroup, $H$, of the symmetry group $G= \text{Sym}_f \cong D_{2p}$. On the other hand, when using random train-test partitions, which are not invariant under any subgroup of $G$, RFM is able to correctly classify test samples.
    \label{claim:breakingsymmetry}
\end{finding} 

Remarkably, for a symmetric train-test partition that yields $0\%$ test accuracy, moving just \textit{one} random training point into the test set (thereby breaking $H$-invariance) can improve test accuracy to $98\%$ in some cases (see Table \ref{tab:accuracy_vs_removed}).
Thus, \textit{less} training data can generalize better when the moved point(s) break symmetry.
To properly understand and characterize these behaviors, we use the symmetry group framework.


\begin{table}[]
\centering
\scalebox{0.85}{
\begin{tabular}{@{}lll@{}}
\toprule
\begin{tabular}[c]{@{}l@{}}Kernel /\\ (Task, Reflection)\end{tabular} &
  \begin{tabular}[c]{@{}l@{}}Num (Random) \\ Points Moved \\ From Train To Test\end{tabular} &
  \begin{tabular}[c]{@{}l@{}}Held-Out\\ Test \\ Accuracy\end{tabular} \\ \midrule
\begin{tabular}[c]{@{}l@{}}Gaussian\\ $(+, s)$\end{tabular} &
  \begin{tabular}[c]{@{}l@{}}0\\ 1\end{tabular} &
  \begin{tabular}[c]{@{}l@{}}0\%\\ 98\%\end{tabular} \\ \midrule
\begin{tabular}[c]{@{}l@{}}Quadratic\\ $(+, s)$\end{tabular} &
  \begin{tabular}[c]{@{}l@{}}0 \\ 100\\ 200\end{tabular} &
  \begin{tabular}[c]{@{}l@{}}0\%\\ 62\%\\ 91\%\end{tabular} \\ \midrule
\begin{tabular}[c]{@{}l@{}}Gaussian\\ $(\cdot, sr^{35})$\end{tabular} &
  \begin{tabular}[c]{@{}l@{}}0 \\ 10\\ 20\end{tabular} &
  \begin{tabular}[c]{@{}l@{}}0\%\\ 13\%\\ 98\%\end{tabular} \\ \midrule
\begin{tabular}[c]{@{}l@{}}Quadratic\\ $(\cdot, sr^{35})$\end{tabular} &
  \begin{tabular}[c]{@{}l@{}}0\\ 100\\ 200\end{tabular} &
  \begin{tabular}[c]{@{}l@{}}0\%\\ 65\%\\ 97\%\end{tabular} \\ \bottomrule
\end{tabular}
}
\caption{Test accuracy of RFM trained on all samples except the fixed points of a reflection $sr^k$, evaluated for different operations modulo $p=61$. The notation $(+,s)$ denotes addition with reflection $s$, and $(\cdot, sr^{35})$ denotes multiplication with reflection $sr^{35}$. Accuracy is reported both before and after randomly removing points from the training set. Under this setting, removing points can improve generalization.}
\label{tab:accuracy_vs_removed}
\end{table}


\paragraph{Symmetry group of modular addition.}
In this section, we characterize the symmetry group for modular addition.
To solve modular addition, we seek to learn the binary function \(f(a,b)=a+b =c\mod p\), with \(a,b,c \in \mathbb{Z}_p\). The symmetry group of \(f\) is:
$$
\begin{aligned}
\text{Sym}_f =  
&\:\{g \in \text{Aut}(\mathbb{Z}^2_p) \mid f(g(a,b))\\
&=f(a,b) \;\;\forall\;\; (a,b) \in \mathbb{Z}^2_p\}
\end{aligned}
$$
\begin{proposition}  The symmetry group of modular addition $f(a,b)=a+b \mod p$ is isomorphic to the dihedral group $D_{2p}$, with a rotation $r(a,b)=(a+1,b-1)$ and a reflection $s(a,b)=(b,a)$. The group has $2p$ elements, which can be classified into
$$
\begin{aligned}
&\text{Rotations:} \:\: r^k(a,b) = (a + k,\; b - k)\\
&\text{Reflections:} \:\:sr^k=(b - k,\;a + k)
\end{aligned}
$$
for $k \in \mathbb{Z}_p$. 
\end{proposition}
We present the proofs that these transformations keep the target function, $f$, invariant and correspond to the dihedral group $D_{2p}$ in Appendix \ref{sec:abelian_invariance_group}, along with the treatment using the \textit{generalized dihedral group} action which can be used for subtraction, multiplication, division, and other Abelian groups.

\paragraph{Fixed points.} The reflections $sr^k$ are \textit{involutions}, that is, when applied twice, they return the same element: $(sr^k)^2(a,b)=(a,b)$. For each $sr^k$, we can partition the data into two sets: the fixed points of $sr^k$ and reflection pairs under $sr^k$. A \textit{fixed point} of $sr^k$ is invariant under the reflection: $sr^k(a,b)=(a,b)$. All points that are not fixed have a pair under that reflection: $\{(a,b),sr^k(a,b)\}$.\\

\begin{proposition} Given all $p^2$ data samples for modular addition:
\begin{enumerate}
\item Every data point is fixed under exactly one reflection $sr^k$.
\item Every data point has a pair under every reflection $sr^{m}$ except for the reflection $sr^k$ under which it is fixed. 
\end{enumerate}
\end{proposition}
We provide the proof of this claim in Appendix \ref{sec:symmetry_classes}. To obtain the expression for the fixed points for addition, it suffices to solve the fixed point equation: 
$$sr^k(a,b)=(b-k,a+k)=(a,b) \implies b = a+k.$$ 
For addition, the fixed points of $sr^k$ are the points $(a,a+k)$ for $a \in \mathbb{Z}_p$. 
In Appendix \ref{sec:fixed_points} we derive the formula for fixed points for the problem of addition of Abelian groups, which can be used to obtain the expression for the fixed points of all four modular arithmetic operations. 

\begin{figure}
\begin{center}
\includegraphics[width=0.92\linewidth]{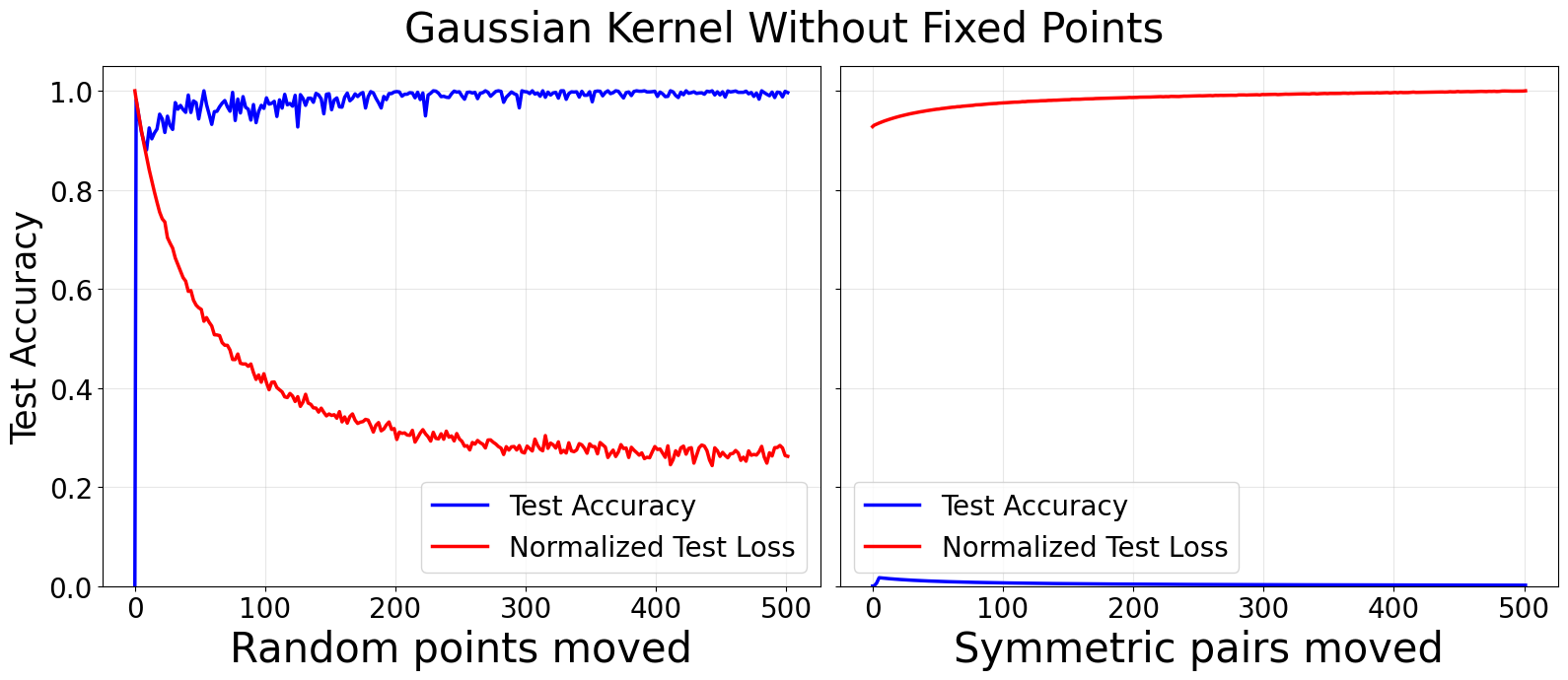}\\
\includegraphics[width=0.92\linewidth]{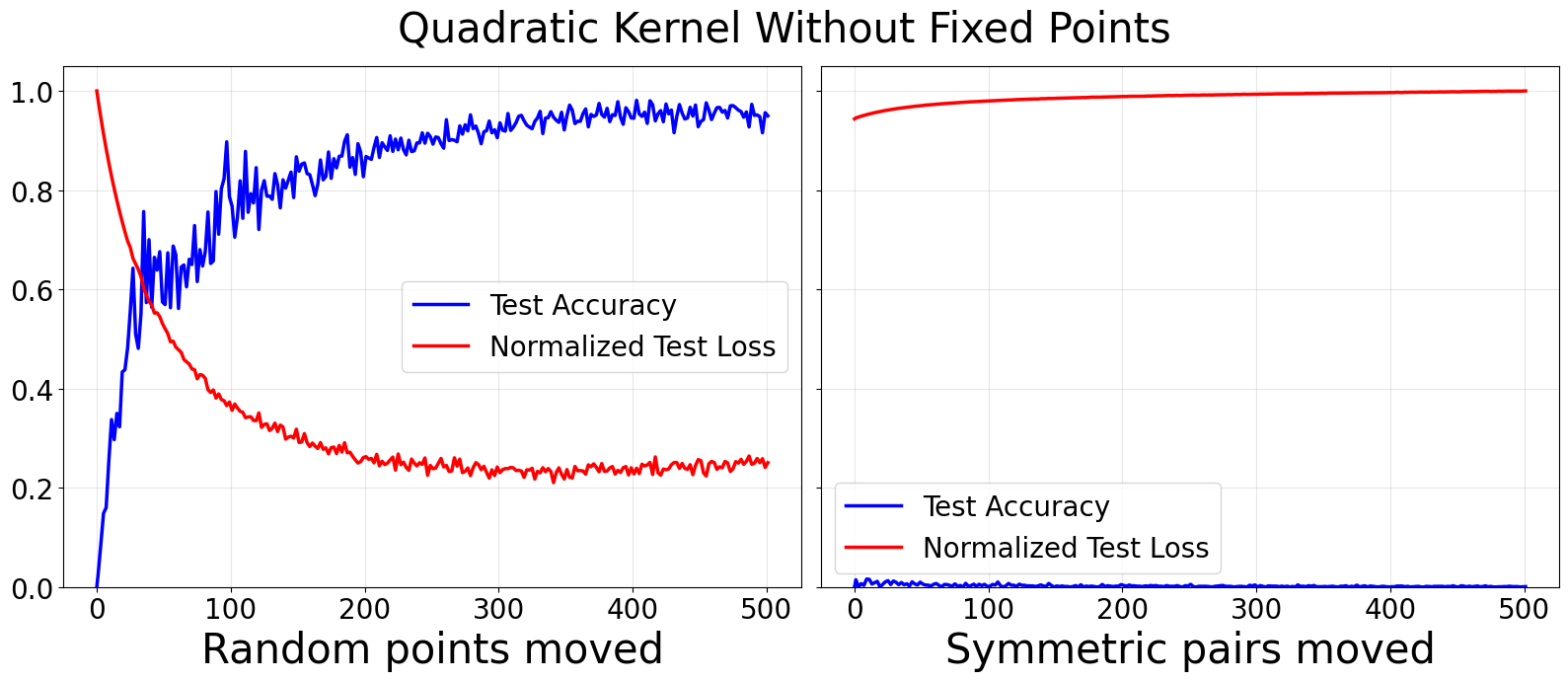}\
\captionof{figure}{For addition modulo $p=53$, we train Gaussian (top row) and quadratic kernels (bottom row) for 60 iterations without the fixed points under reflection $s$, with shape $(a,a)$, and move random points from train to test, which enables generalization. We also move points from train to test by symmetric pairs under the reflection $s$, which doesn't help with generalization. The loss is normalized to map the highest value to 1, and is added to show its evolution. Curves averaged over 5 independent runs.}
\label{fig:fixed_random_points_add}
\end{center}
\end{figure}

\subsection{Partitions that inhibit generalization} \label{sec:partitions}

We proceed to empirically verify Finding \ref{claim:breakingsymmetry}.
We first construct symmetric data partitions that inhibit generalization for RFM.
First, fix any reflection.
For exposition, we choose $s$ in this section and present results for other choices and operations, in Appendix \ref{sec:fixed_point_experiments} (results in this section hold under any such choice). 
We then train and test RFM on the following sets: 
\begin{itemize}
\item Test set ($X_{te}$): all samples encoding the fixed points of $s$, which are the pairs \((a,a)\) for $a \in \mathbb{Z}_p$.
\item Train set ($X_{tr}$): all remaining data samples. 
\end{itemize} 

Note that this training set is $H$-invariant with $H=\{\text{id},s\}$, which means that $h X_{tr}=X_{tr}$ for every $h\in H$. 
We see that RFM does not generalize to the fixed points in the test set, achieving near zero test accuracy (see Table \ref{tab:accuracy_vs_removed}).

It may be tempting to posit that moving fixed points from the test set into the train set would construct a generalizing partition, but we find that is not the case: $H$-invariance is not broken. 
As long as the test set is formed entirely by fixed points under the same reflection, RFM trained on this partition does not generalize. 
This is true for every reflection $sr^k$ and its relevant fixed points.
Empirical verification of this is given in Appendix \ref{sec:fixed_point_experiments}.


Notice that after removing all the fixed points under a reflection $sr^k$ from the train set, we can group all the train samples by pairs under that reflection $sr^k$ (see Appendix \ref{sec:symmetry_classes}), which results in a \textit{symmetric train set} under that reflection. In other words, all orbits under reflection $sr^k$ are complete: there are no points with missing pairs under $sr^k$. 
Since the fixed points under $sr^k$ have pairs under every other $sr^{m}$, this is not the case for any other reflection, as there would otherwise be points left without a pair.

\paragraph{Removing points from the train set asymmetrically. } 
We begin by constructing the symmetric partition as described above, training on all data samples except for the fixed points of $s$, which will form our test set. 
We then move samples at random from the train set to the test set, and train RFM on the resulting set. 
In Figure \ref{fig:fixed_random_points_add} we can see that moving just one point from the train set is enough to make the test accuracy go from near 0\% to 98\% for addition modulo 61 for the Gaussian kernel, which can now correctly predict labels that it previously could not. 
By removing one random sample, we break the $H$-invariance, and the samples in the train set can no longer be grouped by pairs under $s$: the paired point of the removed point is now left without its symmetric pair. 

Table \ref{tab:accuracy_vs_removed} shows how the test accuracy increases before and after removing random points. 
For addition modulo $61$, removing one point is enough to make the Gaussian kernel achieve almost perfect test accuracy, while the quadratic kernel is still around $60\%$ after removing 100 points. 
For multiplication modulo $61$, we remove the fixed points with respect to $sr^{35}$ instead of $s$. The quadratic kernel behaves very similarly, while the Gaussian kernel requires more points to be removed than the additive case (around 20) to achieve near perfect test accuracy.

\paragraph{Removing points symmetrically. } 
For our second experiment we begin with the symmetric partition as we used above, but instead of moving samples at random we move them by pairs. 
That is, if we randomly select and move a training sample \((a,b)\) to the test set, we also move its reflection pair \((b,a)\) under \(s\) (notice it is the same reflection we used to construct the partition). Moving points this way does not break the $H$-invariance, since you can still group all train samples by pairs under $s$.
Indeed, we can see in Figure \ref{fig:fixed_random_points_add} that moving reflection pairs does not have any effect on the test accuracy, which remains near zero. 
No matter how many symmetric pairs we move this way, neither the Gaussian nor the quadratic kernels generalize. 
These results are the same for Abelian groups (see Figure \ref{fig:fixed_random_points_abelian}).

We also saw in Figure \ref{fig:fixed_random_points_add} that, empirically, the quadratic kernel requires more random points to be removed to achieve the same improvement in test accuracy as the Gaussian kernel.
We leave it for future work to explore whether the quadratic kernel is more ``robust" to symmetry breaking, or whether bandwidth choices in the kernel (whether Gaussian or quadratic) can affect the robustness of the model to symmetry breaking.
Nevertheless, the principle of symmetry breaking in this way recovers generalization in both kernels, while on the other hand moving symmetric pairs does not.

\section{FEATURE MATRICES AND GENERALIZATION}\label{features_generalization}
In this section we examine more closely the nature of the features learned by RFM, and relate them to the group structure of the chosen train-test partition.
It was observed in previous work \citep{mallinar2025emergence} that grokking behavior was tied to the learned features being block-circulant (Figure \ref{fig:circulant_features}, first row). 
If we look at the feature matrices (as given by AGOP) learned for the non-generalizing partitions constructed in Section \ref{sec:partitions}, we find a different, but related, structure. 
We now relate these learned feature matrices from our constructed train-test partitions to the elements of the symmetry group.
We first review necessary preliminaries from representation theory.

\paragraph{Representation theory.}

Representation theory examines how the elements of a group can be represented as matrices acting on a vector space, preserving the structure of the group. Formally, a linear representation of a finite group \(G\) is a group homomorphism $\rho:G \to GL(V)$
where \(V\) is a vector space over \(\mathbb{R}\) or \(\mathbb{C}\), and \(GL(V)\) is the general linear group of invertible matrices on \(V\). 

Each group element \(g \in G\) is mapped to an invertible matrix \(\rho(g)\) that satisfies \(\rho(g_1 \star g_2) = \rho(g_1)\rho(g_2)\). Since we are working with one-hot encoded data, we work with a \textit{permutation representation} in which all group elements are encoded as permutation matrices.

\begin{figure}[t]
\centering
\includegraphics[width=0.92\linewidth]{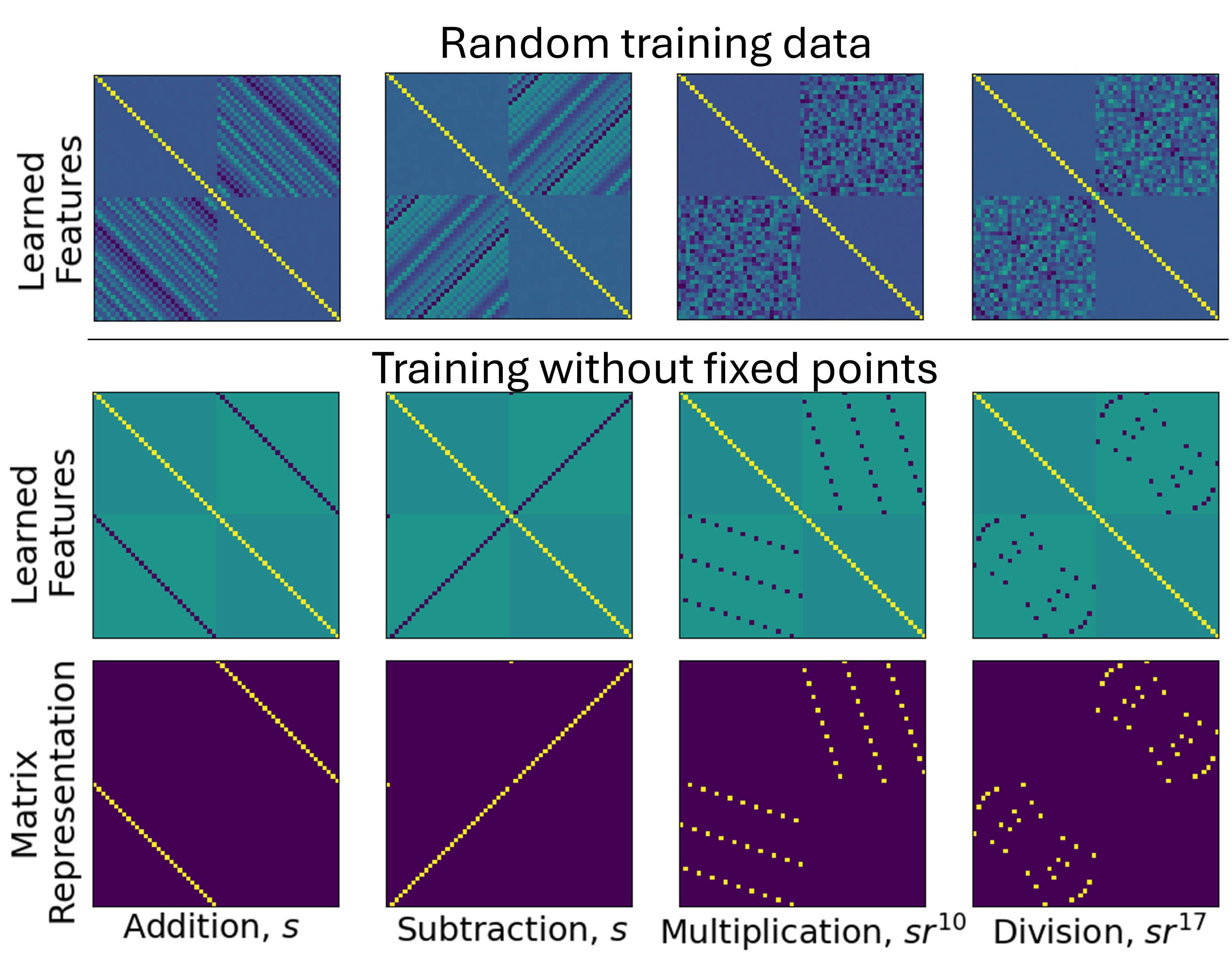}
\captionof{figure}{RFM with a Gaussian kernel on modular arithmetic mod $p = 29$ trained for 60 iterations. The first row shows the features learned when trained on a random partition for modular addition, subtraction, multiplication and division. We compare the AGOP learned by RFM after withholding the fixed points under a given reflection (second row) to the permutation representation of said reflection (third row).}
\label{fig:circulant_features}
\end{figure}

\paragraph{Permutation representation of the symmetry group. } 

We encode group elements $g \in G$ (with $|G|=n$), like \cite{gromov2023grokking, mallinar2025emergence}, as one-hot vectors $e_g \in \mathbb{R}^n$. Input data points are ordered pairs $(a,b) \in G \times G$, represented as the concatenation $x = e_a || e_b \in \mathbb{R}^{2n}$. Group actions on $(a,b)$ then correspond to linear transformations on $x$ implemented by block permutation matrices.

For the symmetry group of Abelian group addition, the generators are
$$
r^x(a,b) = (a \star x,b \star x^{-1}),
\:\:
s(a,b) = (b,a).
$$
Let $R_x \in \mathbb{R}^{n \times n}$ denote the permutation matrix for right multiplication by $x$. Then the associated permutation representation $\Pi: G \to \mathrm{GL}(2n,\mathbb{R})$ is given by
$$
\Pi(r^x) =
\begin{bmatrix}
R_x & 0 \\
0 & R_x^{-1}
\end{bmatrix},
\qquad
\Pi(s) =
\begin{bmatrix}
0 & I_n \\
I_n & 0
\end{bmatrix}.
$$
For subtraction and division, we consider the right-inverse operation problem, $\tilde{f}(a,b)=a\star b^{-1}$ (see Appendix \ref{inverse_composition}), with a slightly different representation (see Appendix \ref{sec:permutation_reps}).

If we train RFM on all data samples save for the fixed points under a reflection $sr^k$, we find that the learned feature matrix at the last iteration, rather than having the full circulant structure learned when trained on a random data partition, only has notable stripes in the points corresponding permutation representations of the reflection $sr^k$ and identity. Note that these two elements form the reflection subgroup $H=\{\text{id},sr^k\}\leq D_{2p}$. In Figure \ref{fig:circulant_features} we compare the learned feature (AGOP) matrices (second row of Figure \ref{fig:circulant_features}) with the theoretical permutation representations (third row of Figure \ref{fig:circulant_features}) of the corresponding reflections in four example cases.
We provide a complete set of such comparisons in Appendix \ref{sec:permutation_reps}.

\paragraph{Other dihedral subgroups.} In the previous experiments we restricted ourselves to the \textit{reflection} subgroups of $D_{2p}$ presented by $H=\langle sr^k\rangle$, which are dihedral subgroups of order 2. By changing $p$ for a non-prime number, the subgroup structure is much richer. We describe the form of the subgroups of the dihedral group in Appendix \ref{sec:subgroups_dihedral}. In short, for a divisor $d\mid p$, we can define the subgroup $\langle r^d, sr^m\rangle$ with a compatible $m$, which has elements $H=\{\text{id},r^d, r^{dk}, \dots, sr^m,sr^{m+d}, sr^{m+dk}, \dots\}$. For example, in addition modulo 32, we have the subgroup $H=\langle r^{16},s\rangle=\{\text{id},r^{16},s,sr^{16}\} \cong D_4$.

If we select the fixed points under every reflection in one of the dihedral subgroups, add them to the test set, and put the rest of samples in the train set, we get an $H$-invariant partition analogous to the \textit{fixed points} partitions we used in our previous experiments. We show this is the case in Appendix \ref{sec:fixed_points}, where we prove the fact that the fixed points under $s$ with same label form a pair under $sr^{16}$ and viceversa. Indeed, the fixed points under $sr^k$ form a pair under $sr^{k-p/2}$, which explains the structure in our partition. We verify this experimentally in Figure \ref{fig:dihedral_subgroups} for the subgroups $H=\langle s\rangle$, $H=\langle r^{16},s\rangle$, $H=\langle r^{8},s\rangle$, $H=\langle r^{4},s\rangle$ (see Appendix \ref{sec:other_subgroups} for additional discussion).

\begin{figure}[t]
    \centering
    \includegraphics[width=1\linewidth]{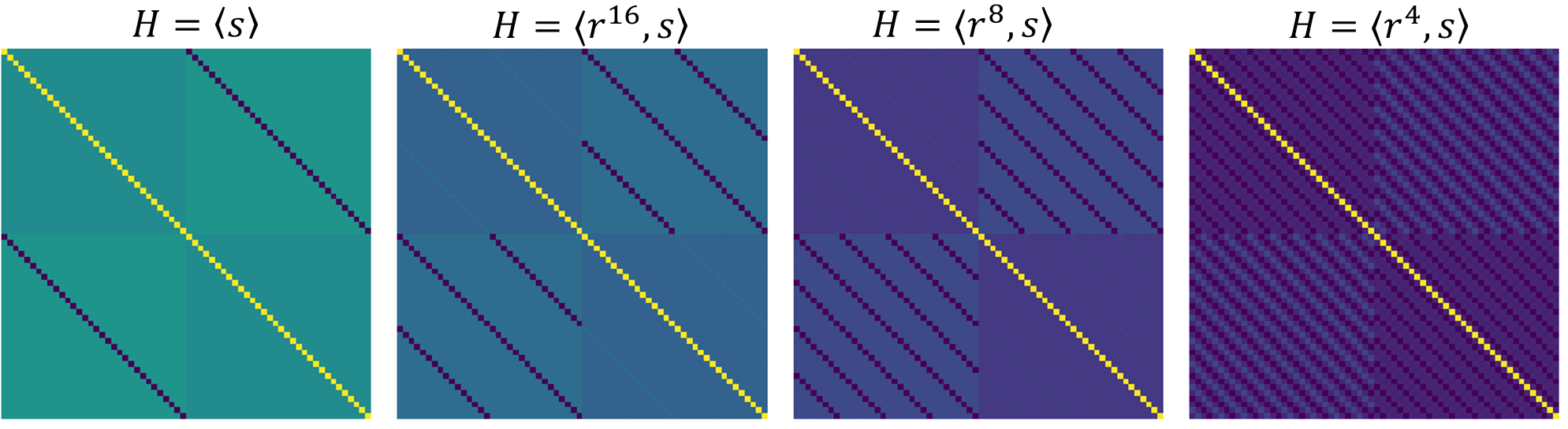}\
    \captionof{figure}{AGOPs learned by RFM with a Gaussian kernel trained on addition mod 32 on all data samples except for the fixed points of all the reflections $sr^k$ of the dihedral subgroups, in order from left to right: $H=\langle s\rangle$, $H=\langle r^{16},s\rangle$, $H=\langle r^{8},s\rangle$, $H=\langle r^{4},s\rangle$. RFM doesn't generalize to the withheld points in any of these settings.}
    \label{fig:dihedral_subgroups}
\end{figure}



\paragraph{RFM generalizes to the orbit of learned symmetries.}
Thus far, we have considered the RFM algorithm with initial features $M_0 = I_{2p}$, meaning that we have not imposed any additional structure on the data at the start of RFM.
By constructing train-test partitions according to symmetries under a fixed reflection, we found that RFM does not generalize to held-out fixed points under that reflection and we found that the learned feature matrices align with the structure of the subgroup $H=\{\text{id},sr^k\}$, exhibiting support primarily on the permutation representations of the subgroup elements. By breaking data symmetry, we recover generalization in RFM.

We now consider the inverse of this setup.
Through our theoretical framework, we show that we can predict the test set points that will be correctly classified by RFM at convergence when we initialize the feature matrix with a prior structure given by any fixed reflection, and use random train-test partitions. Specifically, by forcing the feature structure at initialization to follow a specific reflection, RFM is unable to \textit{escape the symmetry}, and will only generalize to points within the orbit of that reflection.\\

\begin{finding} \label{claim:generalization}
Let $f(a,b)=(a+b) \mod p$ for $p$ prime with symmetry group $ \text{Sym}_f\cong D_{2p}$, let $M_0 \in \mathbb{R}^{2p\times2p}$ be the AGOP learned by RFM when trained on all the data except the fixed points under a reflection $sr^k$. RFM with a Gaussian kernel trained on a random train split $X_{tr} \subseteq\mathbb{Z}_p^2$ and initialized with $M_0$ will:
\begin{enumerate}
\item Memorize the training data $X_{tr}$;
\item Only learn the subgroup $H=\{\text{id},sr^k\}\leq D_{2p}$ imposed by $M_0$.
\end{enumerate}
Then, RFM will only classify correctly the points in $\text{Orbit}_H(X_{tr})$.
\end{finding}

\begin{figure}[t]
    \centering
    \includegraphics[width=0.75\linewidth]{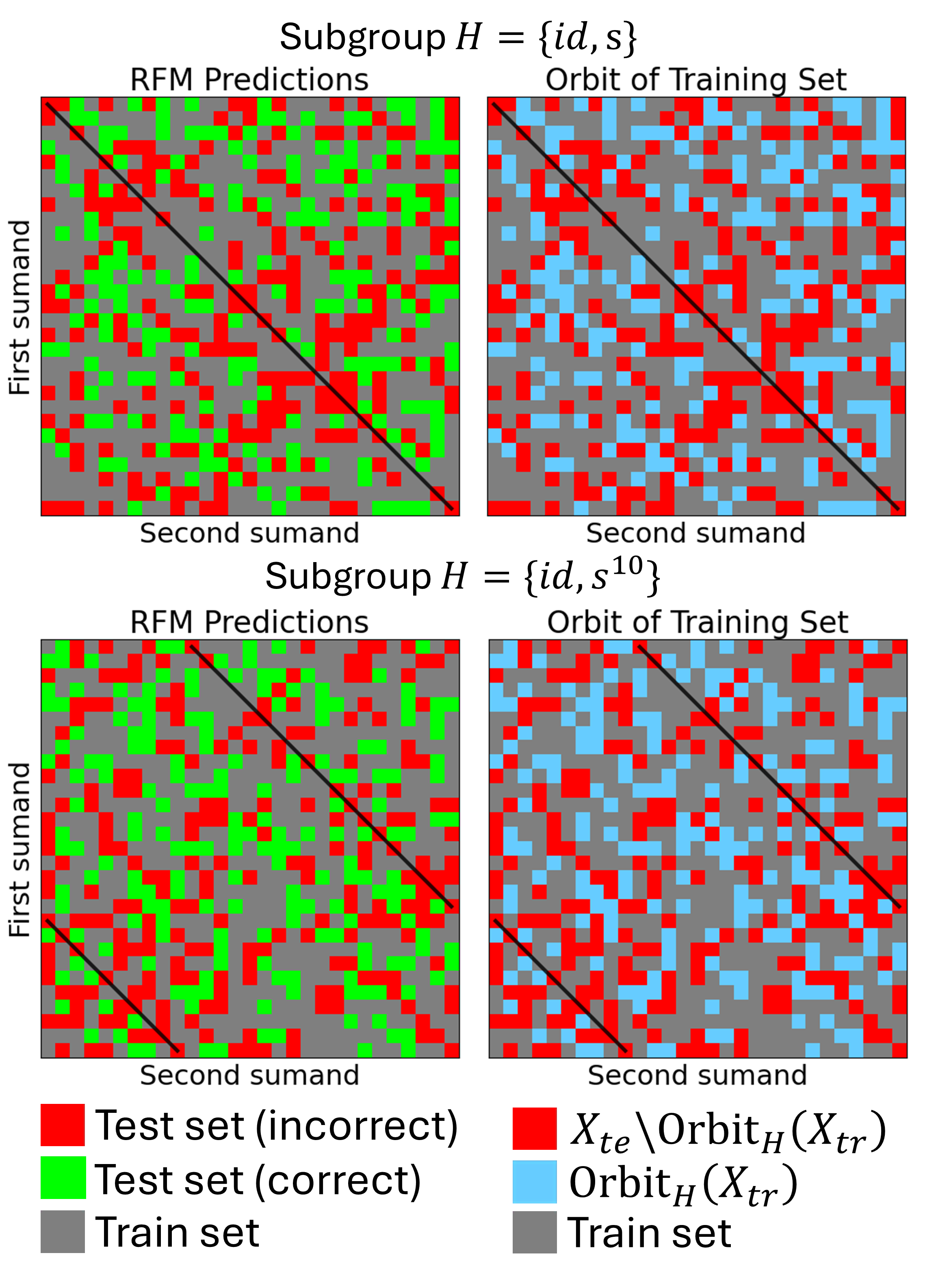}\
    \caption{We train RFM with a Gaussian kernel on addition modulo $p=29$ on 50\% of the data with $M_0$ encoding the subgroup $H=\{\text{id},s\}$ (top row) and $H\{\text{id},s^{10}\}$ (bottom row). The reflection axis is marked by the black line. There is a perfect match between the correct predictions and our theoretical prediction (Finding \ref{claim:generalization}).}
    \label{fig:orbitpredictions}
\end{figure}

Training RFM on a random partition with $M_0$ encoding $H=\{\text{id},sr^k\}$, we can compute $\text{Orbit}_{H}(X_{tr})$ and compare it to the correct guesses in $X_{te}$. We empirically verify this in Figure \ref{fig:orbitpredictions}.
We show the Cayley tables for addition mod 29 where the random training set elements are highlighted in grey, the incorrectly classified elements in the test set are higlighted in red. 

Those correctly classified are highlighted in green in the tables in the first column, while our theoretical predictions on which elements RFM will correctly classify based on the orbit of the subgroup encoded in $H$ are highlighted in blue in the tables in the second column. Empirically, we find a perfect match (100\% precision, 100\% recall) between the theoretical prediction and the empirical results. 

We chose $H=\{\text{id},s\}$ for the tables in the top row. This corresponds to a reflection along the main diagonal. The points along the diagonal are the fixed points, and they are either in the train set or incorrectly classified. The incorrect predictions are symmetric with respect to the diagonal: if $(a,b)$ is incorrectly classified, so will be $(b,a)$ for all $a,b\in\mathbb{Z}_p$.

Confirming Finding \ref{claim:generalization}, if $(a,b)$ is in the train set, $(b,a)$ will be correctly classified  for all $a,b\in\mathbb{Z}_p$.
For the tables in the bottom row, we have $H=\{\text{id},s^{10}\}$, and the same applies, but the reflection axis is the 10th cyclic diagonal (marked by the black line).

\section{RELATED WORK}

\paragraph{Grokking and feature learning in neural networks.}
Recent work has analyzed the dynamics of grokking in neural networks and identified feature learning as a core mechanism behind the phase transition \citep{mlpGrokking22,lyu2023dichotomy,gromov2023grokking,liu2023omnigrokgrokkingalgorithmicdata}.
These works attempt to explain grokking via margin maximization \citep{morwani2024featureemergencemarginmaximization}, mechanistic interpretability of transformers \citep{nanda2023progress}, circuit efficiency \citep{varma2023explaininggrokkingcircuitefficiency}, and transitions from lazy to rich learning \citep{kumar2024grokking}, to name a few.
Much of this work focuses on analyzing the optimization dynamics of neural networks or relies on interpretations of specific neural network architectures.
In our work, we study the group structure of the data, agnostic to neural networks and optimization choices for neural models, through the lens of learned transformations \textit{on the data itself} as understood through RFM and AGOP.


\paragraph{Recursive Feature Machines and AGOP.}

A new direction of research has connected AGOP to the mechanisms of feature learning in neural networks.
These works have shown that RFM recovers many intriguing feature learning phenomena that neural networks exhibit, such as low-rank matrix recovery \citep{LinearRFM}, neural collapse \citep{agopNC}, feature learning in neural networks \citep{BeagleholeNFA,rfm_science,convRFM}, emergence and grokking \citep{mallinar2025emergence}, and the benefit of catapults in the loss landscape \citep{catapultsAGOP}, to name a few.
Thus, RFM and AGOP makes a compelling framework with which to study feature learning for specific tasks.
Our analysis builds on this line of work by studying the way in which RFM learns features from data that follows Abelian group structure.

\paragraph{Group theory and learning algebraic tasks.}


Some prior works have discussed learning algebraic tasks from a group-theoretic perspective.
\citep{mlpGrokking22,morwani2024featureemergencemarginmaximization, stander2024grokkinggroupmultiplicationcosets, shutman2025learningwordsgroupsfusion} study grokking and generalization in neural networks on group operations. \cite{mohamadi2024why} propose that permutation equivariant kernel methods cannot generalize on modular arithmetic tasks, and that neural networks grok when they leave the kernel regime and learn features.
These prior works analyze the group whose operation they are trying to learn, rather than the symmetry group of the target function, whereas our work focuses more on the data and target function than the model itself.
Following a similar path, \citet{perin2025abilitydeepnetworkslearn} analyze a rotated version of the MNIST dataset and argue that neural networks do not generalize to data with algebraic structure unless the symmetries found in the data are imposed on the neural architecture (or equivalent Neural Tangent Kernel).
They additionally relate the ability to generalize to the separation and density of class orbits.

\section{DISCUSSION }

In this work, we focused almost exclusively in the symmetries and properties of the data, and found that data symmetry alone is sufficient to explain and predict the generalization behavior of feature learning kernels. Our findings highlight the importance that training data has on the behavior of a model, and show that data partitions with the same amount of data can yield extremely different results when fed into a model. Under the framework defined in Section \ref{features_generalization}, the results from Section~\ref{sec:breaking} can be simply explained. 
After removing the fixed points of $sr^k$, the learned feature matrices align with the structure of the subgroup $H=\{\text{id},sr^k\}$. Thus, RFM only \textit{learns} the reflection $sr^k$, and since the fixed points of $sr^k$ don't have pairs under the reflection, they aren't in $\text{Orbit}_{H}(X_{tr})$, which explains the 0\% test accuracy. 
Note that in this setup of Section \ref{sec:breaking}, if we initialize RFM with $M_0$ encoding a different reflection $sr^m$ with $m\neq k$, instead of the identity, we recover 100\% test accuracy on the test set of fixed points under $sr^k$ (see Appendix \ref{sec:generalization_experiments}). 

When discussing symmetry breaking, we mention the concept of the train set being invariant under the action of a group $H$. However, $H$-invariance is not the full picture. Breaking $H$-invariance in the data is not necessary to generalize outside of $H$. For example, imposing structure on the model features is enough to restore generalization (see Appendix \ref{sec:generalization_experiments}), and different kernels, parameters or models may exhibit different behavior on the same data partitions. Theoretical proofs and concepts are needed to further formalize the concept of \textit{symmetry traps} presented in this work.

Moreover, the class of problems we study, binary addition in Abelian groups, is characterized by target functions that are fully described by their symmetry groups. In other words, any two samples with the same label can be related by an invertible transformation. Most datasets don't have this property, some target-functions lacking a meaningful symmetry group altogether (MNIST, for example). A deeper understanding of the differences between these regimes may shed some light into why delayed generalization (grokking) is observed in this class of binary operations, and why certain machine learning models struggle to learn symmetries, as observed in \cite{perin2025abilitydeepnetworkslearn}.



The goal of this work was to clarify the issues and set up hypotheses which future work can prove theoretically. Future work may formalize the correspondence between symmetry breaking in RFM (from a permutation-equivariant kernel to a problem-aligned kernel) and the transition from lazy to rich regimes in neural networks, generalizing these results to other models and problems. Additionally, in this work we focused on addition for Abelian groups, since the symmetry group of this problem can be easily described. Understanding the symmetry group for the binary operation of general finite groups problem may extend these ideas to characterize the behavior on other groups.

\subsection*{Acknowledgements}

MT acknowledges the CFIS Mobility Program for the partial funding of this research work, particularly Fundació Privada Mir-Puig, CFIS partners (G-Research, Qualcomm and Semidynamics), and donors of the CFIS crowdfunding program. 
NM acknowledges funding and support for this research from the Eric and Wendy Schmidt Center at The Broad Institute of MIT and Harvard.
The authors would like to additionally thank Daniel Beaglehole and Gil Pasternak for useful discussions during the preliminary phases of this research.

We gratefully acknowledge support from  the National Science Foundation (NSF)
under grants CCF-2112665 and MFAI 2502258, the Office of Naval Research
(ONR N000142412631) and the Defense Advanced Research Projects Agency
(DARPA) under Contract No. HR001125CE020. 
This work used the Delta system at the National Center for Supercomputing Applications through allocation TG-CIS220009 from the Advanced Cyberinfrastructure Coordination Ecosystem: Services \& Support (ACCESS) program, which is supported by National Science Foundation grants \#2138259, \#2138286, \#2138307, \#2137603, and \#2138296.


\bibliography{refs}







\clearpage
\appendix
\thispagestyle{empty}

\onecolumn
\aistatstitle{Supplementary Materials}


\section{Symmetry groups}  \label{sec:abelian_invariance_group}
We study the case of addition of Abelian groups. We have a binary function \(f(a,b)=a \star b =c\), with \(a,b,c \in \mathcal{A}\), where $(\mathcal A,\star)$ is an Abelian group. Here, the group action \(G\) that acts on \(X= \mathcal{A} \times \mathcal{A}\) and keeps \(f\) invariant looks like:
\[
\text{Sym}_f = \{g \in \text{Aut}(\mathcal{A}\times\mathcal A) \mid f(g(a,b)) = f(a,b) \;\;\forall\;\; (a,b) \in \mathcal{A} \times \mathcal{A}\}.
\]
The action we consider is $\mathrm{Dih}(\mathcal A)$, the \textit{generalized dihedral group} of $\mathcal{A}$. This action mimics the dihedral group's behavior, but the cyclic group of rotations is replaced with an arbitrary Abelian group $\mathcal A$. Formally, it is defined as $\mathrm{Dih}(\mathcal A) \;=\; \mathcal A \rtimes C_2 ,\:\: s x s = x^{-1}$, that is, the semidirect product of $\mathcal A$ with $C_2=\{1,s\}$, where $s$ acts by inversion.\\



\begin{lemma}
Let $\mathcal{A}$ be an Abelian group. The symmetry group of the function $f(a,b)=a\star b$ for $a,b\in\mathcal{A}$ is composed by the following maps on $\mathcal A\times \mathcal A$:\\
\begin{enumerate}
    \item Generalized rotation: \(r^x(a,b) = (a\star x,\; b\star x^{-1}), \quad x\in\mathcal A\).
    \item Reflection: \(s(a,b)=(b,a)\).
    \item Reflections: \(sr^x(a,b)=(b\star x^{-1},\;a\star x)\).\\
\end{enumerate} 
\end{lemma}

\begin{proof}
Let $\text{Sym}_f=\{\varphi:\mathcal{A}\times\mathcal{A} \to \mathcal{A} \times \mathcal{A} \:\:\text{bijective}\mid f \circ \varphi = f\}$. That is, $\text{Sym}_f$ consists of bijections of $\mathcal{A}\times\mathcal{A}$ that preserve the value of the group operation $f(a,b)=a\star b$. For each $l\in\mathcal{A}$, the fibre $f^{-1}(l)=\{(a,b) \in \mathcal{A} \times \mathcal{A} \mid a\star b=l\}$ is in bijection with $\mathcal{A}$ via 
$$
\Phi_l:\mathcal{A}\to f^{-1}(l), \quad \Phi_l(x)= (x,x^{-1}\star l),
$$
since $f(x,x^{-1}\star l)=l$. Any $\varphi\in \text{Sym}_f$ permutes fibres, so, given an $l$, it restricts to a permutation of $f^{-1}(l)$, which gives us a bijection $\sigma_l:\mathcal{A} \to \mathcal{A}$ such that
$$
\varphi(x, x^{-1} \star l)=(\sigma_l(x),\sigma_l(x^{-1})\star l)
$$
Write $\varphi(e,e)=(x,x^{-1})$ for some $x \in \mathcal{A}$. Consider $\psi:= r^{x^{-1}} \circ \varphi$, with $r^{x}(a,b)=(a\star x, b \star x^{-1})$. We have $\psi \in \text{Sym}_f$, since $r^x \in \text{Sym}_f$:
\[
\begin{aligned}
f\big(r^x(a,b)\big)
&=f\big(a\star x,\; b\star x^{-1}\big)
=(a\star x)\star(b\star x^{-1})\\
&=a\star x\star b\star x^{-1}
\overset{\text{(commutativity)}}{=}a\star b\star x\star x^{-1}
=a\star b
= f(a,b).
\end{aligned}
\]
It follows that $\psi(e,e)=r^{x^{-1}}(\varphi(e,e))=r^{x^{-1}}(x,x^{-1})=(e,e)$. Thus every $\varphi\in \text{Sym}_f$ can be written as $\varphi=r^x \circ \psi$, with $\psi \in \text{Sym}_f$ fixing $(e,e)$. It suffices to classify the elements of the stabilizer
$$
H:=\{\psi \in \text{Sym}_f \mid \psi(e,e)=(e,e)\}.
$$
Take $\psi \in H$. Thus, we have a family $(\sigma_l)_{l\in \mathcal{A}}$ such that $\psi(x,x^{-1}l)=(\sigma_l(x),\sigma_l(x^{-1})l)$. Since $\psi(e,e)=(e,e)$, we have $\sigma_e(e)=e$. We write $a \star b=ab$ for ease of notation. Define the map 
$\mu : (\mathcal{A} \times \mathcal{A})^2 \to \mathcal{A}^2$
by $\mu((u,v),(u',v')) = (u\star u', v'\star v)$ to model the group operation on $\mathcal{A} \times \mathcal{A}$.

Take $\psi(x,x^{-1})=(\sigma_e(x),\sigma_e(x)^{-1})$ and $\psi(y,y^{-1})=(\sigma_e(y),\sigma_e(y)^{-1})$. Then $\mu((x,x^{-1})(y,y^{-1}))=(xy,y^{-1}x^{-1})$, and $\psi(xy,y^{-1}x^{-1})=(\sigma_e(xy),\sigma_e(xy)^{-1})$. Since both $\psi(xy, y^{-1}x^{-1})$ and $ \mu(\psi(x, x^{-1}), \psi(y, y^{-1}))$ lie in the same fiber $f^{-1}(e)$, and $\Phi_e$ is a bijection, it follows that their first coordinates must be equal. Thus, we have either $\sigma_e(xy)=\sigma_e(x)\sigma_e(y)$ or $\sigma_e(xy)=\sigma_e(y)\sigma_e(x)$.

Hence $\sigma_e$ is a an automorphism $\mathcal{A} \to \mathcal{A}$ or is the map $x \mapsto x^{-1}$. Finally, from $\psi(x,x^{-1})=(\sigma_e(x),\sigma_e(x)^{-1})$ it follows that $\sigma_e$ respects inverses; which gives the only two possibilities:
\[
\sigma_e = \mathrm{id} \qquad \text{or} \qquad \sigma_e : x \mapsto x^{-1}.
\]
Consider the three points:
$$
P=(x,x^{-1}a) \in f^{-1}(a), \quad Q=(e,b) \in f^{-1}(b), \quad R=(x,x^{-1}ab) \in f^{-1}(ab).
$$
Using commutativity of $\mathcal{A}$, it follows that $\mu(P, Q) = R$. Apply $\psi$ and compute 
$$\mu\big(\psi(P), \psi(Q)\big) = \big(\sigma_a(x)\sigma_b(e), (\sigma_b(e)^{-1}b)(\sigma_a(x)^{-1}a)\big)=\big(\sigma_a(x)\sigma_b(e), \sigma_b(e)^{-1}\sigma_a(x)^{-1}ab\big).
$$
Both $\psi(R)$ and $\mu(\psi(P), \psi(Q))$ lie in the same fibre $f^{-1}(ab)$, and by the bijection $\Phi_{ab}$, two elements of that fibre are equal iff their first coordinates coincide. Comparing first coordinates yields the key identity $\ \sigma_{ab}(x) = \sigma_a(x)\sigma_b(e)\ (\forall a,b,x)$, which evaluated at $x=e$ gives $\sigma_{ab}(e) = \sigma_a(e)\sigma_b(e)$, and for $a=e$ gives $\sigma_b(x)=\sigma_e(x)\sigma_b(e)$

\begin{itemize}
    \item If $\sigma_e = \mathrm{id}$, then $\sigma_b(x) = x\sigma_b(e)$. Setting at $x = e$ yields $\sigma_e(b) = \sigma_b(e)$. But $\psi(e, e) = (e, e)$ implies $\sigma_b(e) = e$ for every $b$, so $\sigma_b = \mathrm{id}$ for all $b$. Hence $\psi = \mathrm{id}$.

    \item If $\sigma_e(x) = x^{-1}$, then $\sigma_b(x) = x^{-1} \sigma_b(e)$. Direct verification shows this is exactly the action of the swap $s(a,b)=(b,a)$ on the parametrization: $s(x, x^{-1}b) = (x^{-1}b, x)$, so $\psi = s$.
\end{itemize}

Thus 
\[
H = \{\mathrm{id}, s\}.
\]
Since every $\varphi \in \mathrm{Sym}_f$ decomposes as $\varphi = r^x \circ \psi$ with $x \in \mathcal{A}$ and $\psi \in H$, we obtain the desired description
\[
\boxed{\ \mathrm{Sym}_f = \{r^x \mid x \in \mathcal{A}\} \ \cup\ \{s \circ r^x \mid x \in \mathcal{A}\} \ }.
\]
\end{proof}

\begin{theorem}
Let $\mathcal{A}$ be an Abelian group. The symmetry group of $f(a,b)=a\star b$ for $a,b \in \mathcal{A}$ is isomorphic to the generalized dihedral group:
\[
\text{Sym}_f \;\cong\; \mathrm{Dih}(\mathcal A) \;=\; \mathcal A \rtimes C_2.
\]
\end{theorem}
\begin{proof}
The subgroup generated by the $r^x$’s is isomorphic to $\mathcal A$: indeed $r^x \circ r^y = r^{x\star y}$ and $(r^x)^{-1}=r^{x^{-1}}$ for any $x,y \in \mathcal{A}$. Thus $\langle r^x: x\in\mathcal A\rangle \cong \mathcal A$. We have $s^2=\mathrm{id}$, and conjugation by $s$ acts by inversion:
\[
s r^x s (a,b) = s(b\star x, a\star x^{-1}) = (a\star x^{-1}, b\star x) = r^{x^{-1}}(a,b).
\]
Thus $s r^x s = r^{x^{-1}}$. These are exactly the relations defining the semidirect product $\mathcal A \rtimes C_2$, where $C_2$ acts by inversion. Therefore the group formed by the elements $r^x$ and $sr^x$ is isomorphic to $\mathrm{Dih}(\mathcal A)$. Since we already proved these maps generate a subgroup of $\mathrm{Sym}_f$, and they realize the relations of $\mathrm{Dih}(\mathcal A)$, we conclude $\text{Sym}_f \;\cong\; \mathrm{Dih}(\mathcal A).$
\end{proof}

\subsection{Modular arithmetic}
When $\mathcal A$ is cyclic, the generalized dihedral group reduces to the classical dihedral group. This is the case for modular arithmetic operations like addition and multiplication modulo $p$.\\

\begin{proposition}
If $\mathcal A = C_n$, where $C_n$ is the finite cyclic group of order $n$, then
\[
\mathrm{Dih}(\mathcal A) \;\cong\; D_{2n},
\]
the dihedral group of order $2n$.
\end{proposition}

\begin{proof}
Take $\mathcal A = \langle x \rangle \cong C_n$. Then $\mathrm{Dih}(\mathcal A) = \mathcal A \rtimes C_2$ has generators $r$ and $s$ with relations
\[
r^n=e,\quad s^2=e,\quad srs=r^{-1}.
\]
This is precisely the presentation of the classical dihedral group $D_{2n}$. Hence $\mathrm{Dih}(C_n) \cong D_{2n}$.
\end{proof}
For example, consider $\mathcal A = \mathbb Z_p$ with modular addition. The rotations are $r^k(a,b)=(a+k,\; b-k)$, taking the results modulo $p$, the main reflection is $s(a,b)=(b,a)$, and the relations above show the group is $D_{2p}$.
\[
D_{2p} = \langle r, s \mid r^p = e, \: s^2 = e, \: s r s = r^{-1} \rangle.
\] 
\subsection{Inverse-variant group action} \label{inverse_composition}
We want to work with modular subtraction and division, which the use \textit{right-inverse operation}  $\tilde{f}(a, b) = a \star b^{-1}$ instead of the standard (direct) group operation $f(a, b) = a \star b$. To describe the symmetries of those functions, we define an associated \textit{inverse-variant group action}. This action, denoted $\tilde{G}$, mirrors the original generalized dihedral group action $G$, but acts compatibly with the right-inverse operation. In the additive setting, this corresponds to subtraction; in the multiplicative setting, to division. Let \(G\) be the \textit{generalized dihedral group} of $\mathcal A$ acting on \(X=\mathcal{A} \times \mathcal{A}\), and define the inverse-variant group action \(\tilde{G}\) on \(X\) as:\\

\begin{enumerate}
    \item Generalized rotation: \(\tilde{r}^x(a,b) = (a \star x,\; b \star x)\)
    \item Reflection: \(\tilde{s}(a,b) = (b^{-1},\; a^{-1})\)
    \item Reflections: \(\tilde{s} \tilde{r}^x(a,b) = ((b \star x)^{-1},\; (a \star x)^{-1})\)\\
\end{enumerate}

To make the relationship between $G$ and $\tilde{G}$ explicit, we define a generator mapping between the two groups \(\Phi : \tilde{G} \to G\), connecting the inverse-variant action to the direct one:
\[
\Phi(\tilde{r}^x) = \Delta_x \circ r^x, \quad \text{where } \Delta_x(a,b) := (a,\; b \star x^2)
\]
\[
\Phi(\tilde{s}) = \iota \circ s, \quad \text{where } \iota(a,b) := (a^{-1},\; b^{-1})
\]
where \(x,a,b \in \mathcal{A}\). That is, the inverse-variant rotation is equivalent to applying the original rotation followed by a coordinate-wise scaling of the second component; and the inverse-variant reflection is inversion applied after permutation.

\section{Subgroup structure of $D_{2n}$}\label{sec:subgroups_dihedral}
To study the orbits under the action of the dihedral group $D_{2n}$, we need to study its subgroup structure. Using standard notation, the presentation of the dihedral group is:
$$
D_{2n}=\langle a, x \mid a^n=x^2=e, \: xax^{-1}=a^{-1}\rangle.
$$
The group has two generators, $a$, the rotation of order $n$, and $x$, the reflection of order 2. All subgroups fall into two distinct types:
\begin{enumerate}
\item Cyclic (rotation) subgroups $\langle a^d\rangle \cong \mathbb{Z}_{n/d}$. There is one such subgroup for each possible divisor $d\mid n$, which in total are $\tau(n)=\text{number of positive divisors of } n$. 
\item Dihedral subgroups $\langle a^d, a^rx\rangle \cong D_{2n/d}$, for $d\mid n$ and $0\leq r <d$. The total number of subgroups of this type is $\sigma(n)=\text{the sum of positive divisors of }n$.
\end{enumerate} 
Notice that the dihedral subgroups include the \textit{reflection subgroups} $\langle a^rx \rangle$ of order 2 when $d=n$, since $a^n=\text{id}$. When $n$ is prime, all the cyclic subgroups $\langle a^d\rangle$ are isomorphic to $\mathbb{Z}_n$, and the only dihedral subgroups are the \textit{reflection subgroups} mentioned before. The subgroup structure of the generalized dihedral group of an Abelian group $\mathcal{A}$, denoted $\text{Dih}(\mathcal{A})$, depends on the Abelian group. For the purposes of this work, it suffices with knowing that the reflections also form the same \textit{reflection subgroups} of order 2 as the standard dihedral case. Note that we denote the rotation as $r$ and the reflection as $s$ thorough the rest of the work.

\section{Symmetry Classes} 
\label{sec:symmetry_classes}
In this appendix we explore the structure of symmetry classes induced by the generalized dihedral group action. 
Throughout, let $(\mathcal A,\star)$ be an Abelian group and let $G=\mathrm{Dih}(\mathcal A)$ act on $X=\mathcal A \times \mathcal A$ as described in Section~\ref{sec:abelian_invariance_group}. 
The goal is to understand how this action partitions $X$ into equivalence classes, how these classes are internally structured by sub-actions, and how this specializes in modular arithmetic.

\paragraph{Symmetry classes as orbits.}
The orbit of a point \(x \in X\) under the group action \(G\) is defined as $\text{Orb}_G(x) = \{g  x \mid g \in G\}.$ Two points belong to the same symmetry class if they lie in the same orbit, i.e., if one can be transformed into the other by some group element. Since the group action preserves the operation $f(a,b)=a\star b$, given a sample $(a,b)\in X$, another sample $(a',b')$ belongs to the same $G$-orbit if and only if it has the same \emph{label} $l = f(a,b)=a\star b.$ Furthermore, all points with label $l$ belong to the same symmetry class, and thus we may  index the orbit by its label:
\[
\text{Orb}_G((a,b)) = O_l = \{(a',b')\in \mathcal A \times \mathcal A \mid a'\star b' = l\}.
\]
Thus $X$ is partitioned into disjoint classes $\{O_l : l\in\mathcal A\}$, each class collecting all pairs that yield the same product/sum under $\star$. 

\subsection{Sub-orbits and internal structure}
Within each symmetry class $O_l$, we may look at the finer partition induced by a subgroup of $G$. 
Given a subgroup $H\leq G$, it partitions $O_l$ into \emph{sub-orbits}:
\[
\text{Orb}_H((a,b)) = \{ h(a,b) \mid h\in H\}.
\]
This yields a decomposition
\[
O_l = \bigsqcup_{i} \text{Orb}_H(x_i),
\]
with one representative $x_i$ chosen from each sub-orbit. 
We refer to this as the \emph{internal orbit partition} of $O_l$ with respect to $h$:
\[
\mathcal{P}_h(O) := \left\{ \text{Orb}_H(x') \mid x' \in O \right\} / \sim
\]
where $\sim$ identifies equivalent orbits (orbits that contain the same elements):
\[
\text{Orb}_H(x') \sim \text{Orb}_H(x'') \iff \text{Orb}_H(x') = \text{Orb}_H(x'')
\]

\paragraph{Reflection subgroups.}
Every reflection $sr^x$ forms a subgroup $H=\{\text{id}, sr^x\}$, since they are elements of order 2 and their action is an involution: $(sr^x)^2 = \mathrm{id}.$ Thus each sub-orbit defined by a subgroup $H=\{\text{id}, sr^x\}$ has either size 1 (a fixed point) or size $2$ (a reflection pair). This provides a canonical “mirror symmetry” structure within each $O_l$.\\

\begin{proposition} Given all $n^2$ data samples for the problem of addition of an Abelian group $\mathcal{A}$ of order $n$:
\begin{enumerate}
\item Every data point is fixed under exactly one reflection $sr^x$.
\item Every data point has a pair under every reflection $sr^{m}$ except for the reflection $sr^x$ under which it is fixed. 
\end{enumerate}
\end{proposition}

\begin{proof}
Let $\mathcal A$ be an Abelian group. Recall $sr^x(a,b)=(b \star x^{-1},\;a \star x)$.

\emph{Fixed point.}
We solve the fixed point equation, $sr^x(a,b)=(a,b)$. The equalities $b\star x^{-1}=a,\: a \star x=b$ are equivalent to $x=a^{-1} \star b$, which is the unique solution.

\emph{Reflection pairs.}
If $x \neq a^{-1} \star b$, by uniqueness we know that $(a',b'):=sr^x(a,b)=(b \star x^{-1},a\star x)\neq (a,b)$. Because $sr^x$ is an involution, $sr^x(a',b')=(a,b)$, so $\{(a,b),(a',b')\}$ is a 2-element orbit (a reflection pair). Thus every non-fixed element lies in a 2-cycle under $sr^x$.
\end{proof}

\subsection{Fixed points} \label{sec:fixed_points}
Let \(G\) be a group acting on \(X\). The fixed points of a group element \(g \in G\), are points that are invariant under the action of \(g\). Formally, the fixed points of an element \(g\) are $\text{Fix}(g) = \{x\in X \mid g  x=x\}$. For reflections $sr^x$, we can find the expression for fixed points:
\[
sr^x(a,b)=(b \star x^{-1},\: a\star x) = (a,b) \quad \Rightarrow \quad b = a \star x.
\]
Which means that the fixed points of \(sr^x\) have the shape
\[
\text{Fix}(sr^x) = \{(a,\; a\star x) : a\in\mathcal A\}.
\]
In the case of inverse-variant actions, analogous formulas hold. 
Using its structural definition $\tilde{sr}^x=\iota \circ s \circ (\Delta_x \circ r^x)$ results in $\tilde{sr}^x(a,b)=\iota(s(\Delta_x(r^x(a,b)))) = ((b\star x)^{-1},\:(a\star x)^{-1})$, which gives us
\[
\text{Fix}(\tilde{sr}^x) = \{(a,\; (a\star x)^{-1}) : a\in\mathcal A\}.
\]

\begin{center}
\begin{tabular}{ccccc}
Operation & Addition & Multiplication & Subtraction & Division \\
\toprule
Fixed points
& $(a,\; a+k)$ 
& $(a,\; a\cdot k)$ 
& $(a,\; -a-k)$ 
& $(a,\; 1/ak)$ \\
\end{tabular}
\end{center}

\paragraph{Fixed points for modular addition. } We give a brief description of the fixed points in this setting. For addition modulo $p$, a point $(a,b)$ will be fixed under reflection $sr^k$ if
$$
sr^k(a,b)=(b-k,a+b)=(a,b) \implies k=b-a
$$
and will have label $l=a+b$. This gives us the solutions $a=\frac{l-k}{2},b=\frac{l+k}{2}$. When $p$ is odd, the solution is unique, which means for each label theres exactly one fixed point under every reflection. However, if $p$ is even, 2 no longer has a multiplicative inverse.

For $p$ even, the equations only have a solution if $l$ and $k$ have the same parity: either they are both odd, or they are both even. Indeed, $a=\frac{l-k}{2},b=\frac{l+k}{2}$ requires $l-k$ and $l+k$ to be divisible by 2. Additionally, there are now two possible fixed points for a valid combination of $l$ and $k$: 
$$
(a_1,b_1)=(\frac{l-k}{2} \mod p,\frac{l+k}{2} \mod p) \quad (a_2,b_2)=(\frac{l-k+p}{2} \mod p,\frac{l+k+p}{2} \mod p)
$$
For example, for $p=32$, we have two fixed points under $k=0$ with label $l=0$: $(0,0), (16,16) \mod 32$. We now want to find the reflection $sr^m$ that swaps between the two fixed points under $k$ (we omit the mod $p$ notation for clarity):
$$
\begin{aligned}
sr^m(\frac{l-k}{2},\frac{l+k}{2})=
&(\frac{l+k}{2}-m,\frac{l-k}{2}+m)=(\frac{l-k+p}{2},\frac{l+k+p}{2}) \implies \\
&\frac{l+k}{2}-m=\frac{l-k+p}{2}; \: \frac{l-k}{2}+m=\frac{l+k+p}{2}
\end{aligned}
$$
Which gives us the solution $m=k-\frac{p}{2}$. Moreover, the fixed points under reflection $sr^m$ are swapped by $sr^k$, since $m-\frac{p}2 = k - \frac{p}2-\frac{p}2=k$.

\subsection{Example: Prime modular addition}  \label{sec:modular_addition}
Let $\mathcal A=\mathbb Z_p$ with prime $p$, and consider $f(a,b)=a+b\pmod p$. 
Here $G\cong D_{2p}$, the classical dihedral group. For readability, we omit the explicit modular notation, with the understanding that all operations take place in \(\mathbb{Z}_p\). 

\paragraph{Symmetry classes.} For each $l\in \mathbb Z_p$, the symmetry class is $O_l = \{(a,b)\in\mathbb Z_p^2 : a+b=l\pmod p\}.$ Each $O_l$ has size $p$, and the $p$ classes partition $\mathbb Z_p^2$.

\paragraph{Cyclic sub-orbits.} Since $p$ is prime, each rotation $r^k$ generates the complete cyclic subgroup, $\langle r^k\rangle \cong \mathbb{Z}_p$ for any $k\in\mathbb{Z}_p$. Furthermore, given a point $(a,b)$, the orbit of the cyclic subgroup $\langle r \rangle$ is the class $O_l$. Thus, the sub-orbit of the cyclic subgroup $\langle r \rangle $ equals the full class $O_l$.

\paragraph{Reflection sub-orbits.} Since $p$ is prime, the only dihedral subgroups are the reflection subgroups $\langle sr^k\rangle$. Each reflection subgroup $\langle sr^k\rangle$ partitions $O_l$ into:
\begin{itemize}
\item One fixed point $(a,b) = \big(\tfrac{l-k}{2}, \tfrac{l+k}{2}\big)$, using invertibility of $2$ mod $p$.
\item $(p-1)/2$ reflection pairs of the form $(a,b)$ and $(b-k,a+k)$ swapped by $sr^k$.
\end{itemize}
Each reflection therefore defines a unique partition of $O_l$, and each point in $O_l$ is the fixed point of exactly one reflection. This follows from the fact that each $O_l$ has $p$ elements; removing the unique fixed point leaves $p-1$, which is even, hence splits into exactly $(p-1)/2$ disjoint 2-cycles.

\section{Fixed point experiments}\label{sec:fixed_point_experiments}

We seek to generalize our results to the addition operation of Abelian groups. Let \(\mathcal{A}\) be an Abelian group, meaning \(a\star b = b \star a\) for all \(a,b \in \mathcal{A}\). We now state the Fundamental Theorem of Finite Abelian Groups. Let \( \mathcal{A} \) be a finite Abelian group. Then there exist positive integers \( n_1, n_2, \dots, n_r \) such that
\[
\mathcal{A} \cong C_{n_1} \times C_{n_2} \times \cdots \times C_{n_r},
\]
where each \( C_{n_i} \) is a cyclic group of order \( n_i \), and one of the following conditions holds:
\begin{itemize}
    \item[(1)] (Primary decomposition form) Each \( n_i \) is a power of a prime, and this decomposition is unique up to isomorphism and ordering.
    \item[(2)] (Invariant factor form) The integers satisfy \( n_1 \mid n_2 \mid \cdots \mid n_r \), and the decomposition is unique up to isomorphism.
\end{itemize}

We consider the problem of learning a binary function \( f: \mathcal{A} \times \mathcal{A} \to \mathcal{A} \), where \( \mathcal{A} \) is a finite Abelian group and \(f(a, b) = a \star b,\) where \( \star \) denotes the group operation in \( \mathcal{A} \). Since every finite Abelian group is isomorphic to a direct product of cyclic groups, we write \(\mathcal{A} \cong \mathbb{Z}_{n_1} \times \mathbb{Z}_{n_2} \times \cdots \times \mathbb{Z}_{n_k}.\) Accordingly, each group element \( a \in \mathcal{A} \) can be written as a tuple \( a = (a_1, a_2, \ldots, a_k) \), where \( a_i \in \mathbb{Z}_{n_i} \). The group operation is performed componentwise: $a \star b = (a_1 + b_1 \bmod n_1,\ a_2 + b_2 \bmod n_2,\ \ldots,\ a_k + b_k \bmod n_k).$

\subsection{Moving fixed points}

We proceed to experimentally verify the claims from Section \ref{sec:breaking}. For arithmetic modulo $p$, we begin by designing the same partition, in which the test set ($X_{te}$) contains all $p$ samples (or $p-1$ if we're dealing with multiplication or division) encoding the fixed points under a reflection $sr^k$, and the train set ($X_{tr}$) contains the remaining data samples. RFM with Gaussian and quadratic kernels don't generalize on this partition. Moving fixed points from the test set to the train set doesn't result in a generalizing partition, and we verify that for the Gaussian kernel in Figure \ref{fig:moving_fixed_points} (A) for addition with reflection $s$ and multiplication with refletion $sr^{13}$. The $x$-axis represents the number of fixed points moved from the test set to the train set. There is no sequencial dependence in the removal process, meaning that in each step  we randomly select $x$ fixed points from the test set. This adds robustness to the experiment by ensuring that each set of points is chosen independently, compared to a standard method where the $x$-th point would be removed based on the previous step (e.g., removing the $x-1$ points from the previous step and then an additional random one).

This ties with our observation that, when holding out two points, RFM generalizes \textit{only} if the pair of points is not invariant under a common reflection. In Figure \ref{fig:moving_fixed_points} (B) we compare the features (AGOP) learned by RFM with a Gaussian kernel on the full dataset except for two random points four different times. In the leftmost AGOP, both removed points are fixed under the same reflection $s$, while in the rest the points are fixed by different reflections, which can be observed in the diagonal stripes. 

\begin{minipage}{1\textwidth}
    \centering
    \includegraphics[width=1\linewidth]{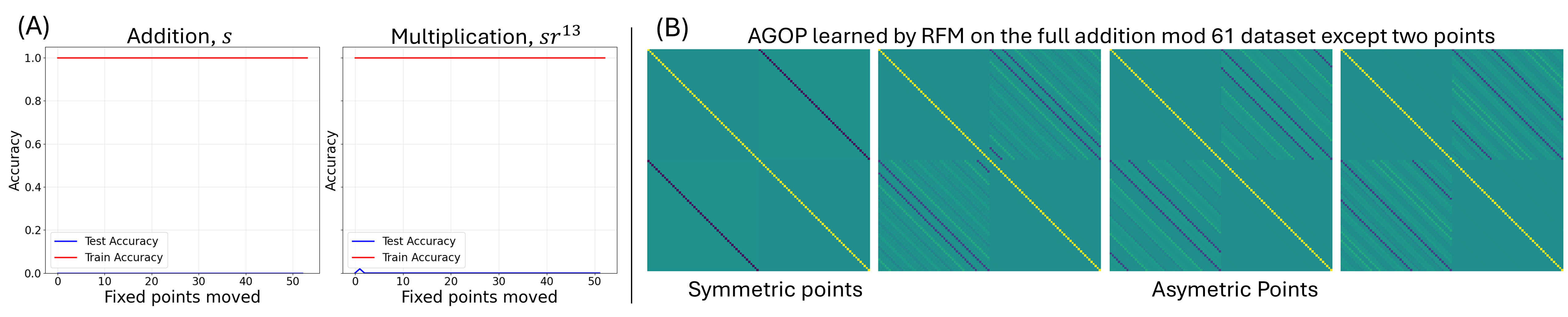}\
    \captionof{figure}{(A) For addition with reflection $s$ (left) and multiplication with reflection $sr^{13}$ (right) mod 53, we train RFM with a Gaussian kernel for 60 iterations without the fixed points under that reflection, and move fixed points from test to train. This doesn't recover generalization. (B) AGOPs learned by RFM with a Gaussian kernel on addition modulo 61 on the full dataset except two random samples. In the first image, both withheld points are fixed under the same reflection $s$ (non-generalizing partition), while in the rest the withheld points are fixed under different reflections (generalizing partition).}
    \label{fig:moving_fixed_points}
\end{minipage}

\subsection{Moving random points and symmetric pairs}

We repeat the experiments from Section \ref{sec:breaking}. First, we construct the same partition: the test set ($X_{te}$ contains all the fixed points under a reflection $sr^k$, and the train set ($X_{tr}$) contains the remaining data samples, for an arbitrary reflection and operation or Abelian group. In Figure \ref{fig:fixed_points_mult_div}, we experimentally verify that, starting from that partition, moving points at random from the train set to the test set results in a generalizing partition both in quadratic and Gaussian kernels for multiplication mod 61 with reflection $sr^{35}$ (left) and division mod 53 with reflection $sr^{27}$ (right). The $x$-axis represents the number of random points moved from the train set to the test set. Similar to the previous subsection, there is no continuity in the removal process, meaning that in each step we randomly select $x$ points from the train set set, which makes the selected points independent from the previous step.

\begin{minipage}{1\textwidth}
    \centering
    \includegraphics[width=1\linewidth]{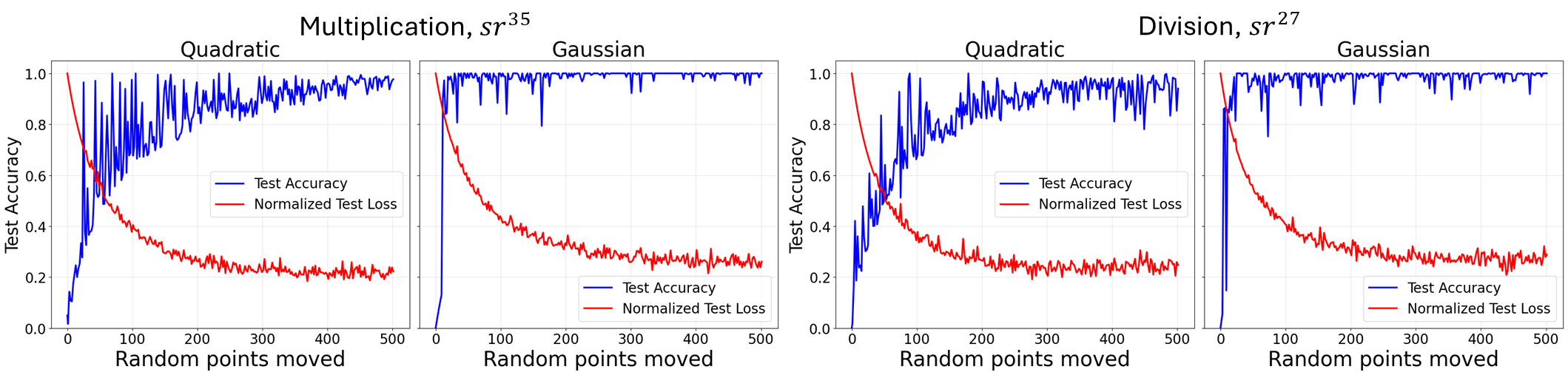}\
    \captionof{figure}{\textbf{For multiplication mod 61} (left), we train RFM with Gaussian and quadratic kernels for 60 iterations without the fixed points under reflection $sr^{35}$, and move random points from train to test, which enables generalization. We do the same \textbf{for division mod 53} (right) under reflection $sr^{27}$. The loss is normalized to map the highest value to 1, and added to show its evolution.}
    \label{fig:fixed_points_mult_div}
\end{minipage}

In Figure \ref{fig:fixed_random_points_abelian}, we repeat the experiment for the Abelian group $\mathcal{A} \cong C_5 \times C_{11}$ with the reflection $sr^{(3,2)}$ (remember reflections are indexed by elements of the Abelian group) for Gaussian (left) and quadratic (right) kernels, and also include the experiment in which, rather than moving points at random, we move symmetric pairs under the chosen reflection, which doesn't result in a generalizing partition. Like in the previous experiments, there is no continuity in the removal process.

\begin{minipage}{1\textwidth}
    \centering
    \includegraphics[width=0.49\linewidth]{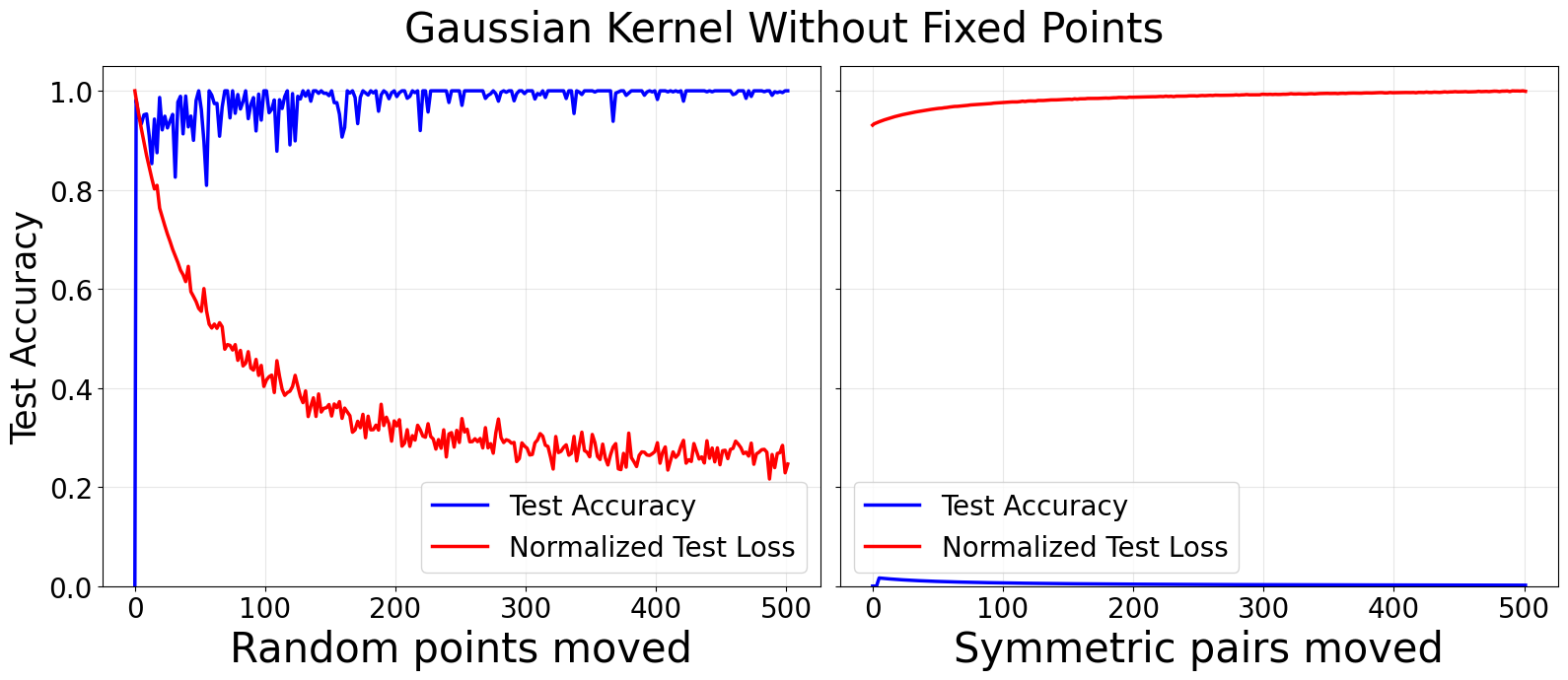}\
    \includegraphics[width=0.49\linewidth]{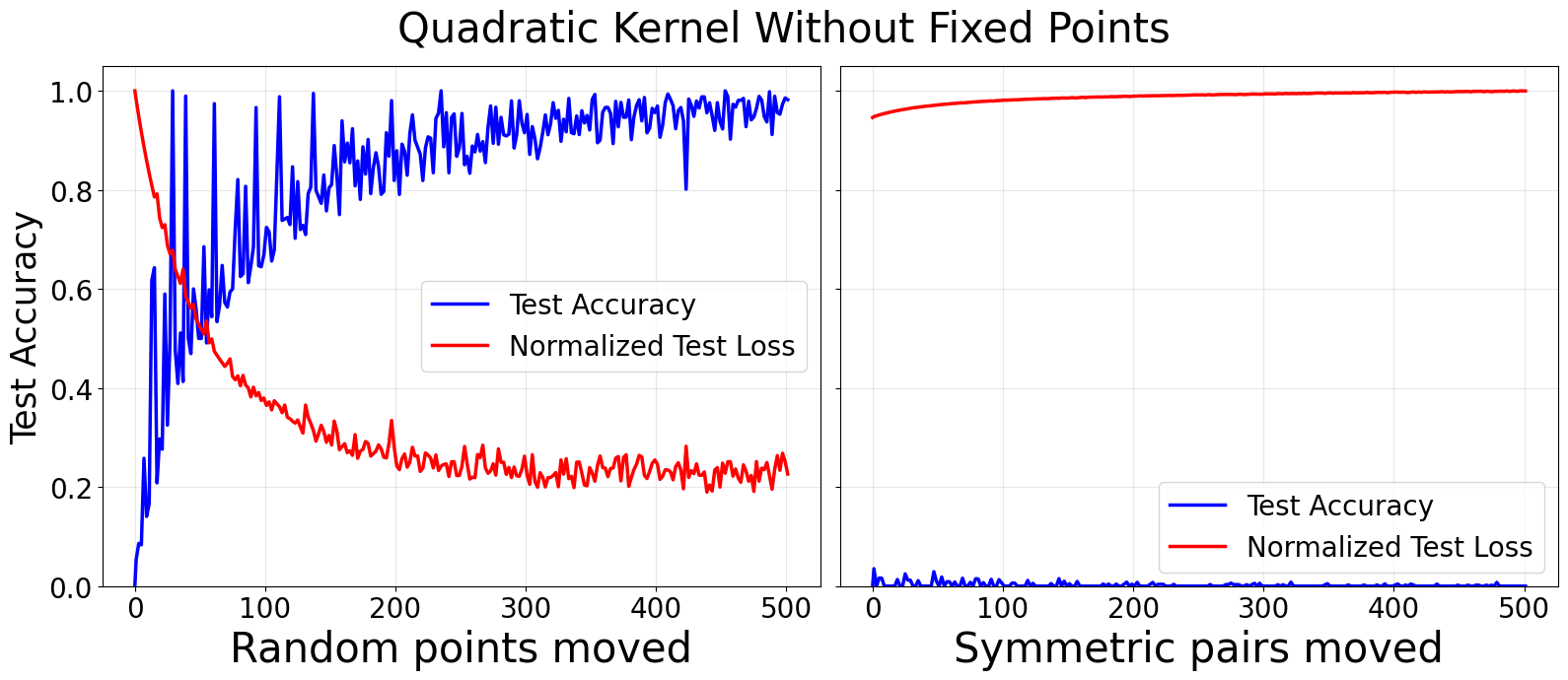}
    \captionof{figure}{\textbf{For the Abelian group $\mathcal{A} \cong C_5 \times C_{11}$}, we train RFM with Gaussian (left) and quadratic kernels (right) for 60 iterations without the fixed points under reflection $sr^{(3,2)}$, and move random points from train to test, which enables generalization. We also move points from train to test by symmetric pairs under the reflection $sr^{(3,2)}$, which doesn't help with generalization. The loss is normalized to map the highest value to 1, and added to show its evolution.}
    \label{fig:fixed_random_points_abelian}
\end{minipage}

\subsection{Other subgroups of $D_{2p}$}\label{sec:other_subgroups}
In the previous experiments we restricted ourselves to the \textit{reflection} subgroups of $D_{2p}$ presented by $H=\langle sr^k\rangle$. We seek to repeat some of the experiments for other subgroups, mainly the dihedral subgroups described in Section \ref{sec:subgroups_dihedral}. For a divisor $d\mid p$, we can define the subgroup $\langle r^d, sr^m\rangle$ with a compatible $m$, which has elements $H=\{\text{id},r^d, r^{dk}, \dots, sr^m,sr^{m+d}, sr^{m+dk}, \dots\}$. For example, in addition modulo 32, some of the subgroups are:
\begin{enumerate}
\item $H=\langle r^{16},s\rangle=\{\text{id},r^{16},s,sr^{16}\} \cong D_4$.
\item $H=\langle r^{8},s\rangle=\{\text{id},r^{8},r^{16},r^{24},s,sr^{8},sr^{16},sr^{24}\} \cong D_8$.
\item $H=\langle r^{4},s\rangle=\{\text{id},r^{4},r^{8},r^{12},r^{16},r^{20},r^{24},r^{28},s,sr^{4},sr^{8},sr^{12},sr^{16},sr^{20},sr^{24},sr^{28}\} \cong D_{16}$.
\end{enumerate}

If we select the fixed points under every reflection in those subgroups, build the test set with them, and build the train set with remaining samples, we get an $H$-invariant partition analogous to the \textit{fixed points} partitions we used in our previous experiments. This can be checked from our description of the fixed points in the non-prime setting in Section \ref{sec:fixed_points}, which shows the fact that the fixed points under $s$ with same label form a pair under $sr^{16}$ and viceversa. Indeed, the fixed points under $sr^k$ form a pair under $sr^{k-p/2}$, which explains the structure in our partition. We verify this experimentally in Figure \ref{fig:dihedral_subgroups2} for the subgroups $H=\langle s\rangle$, $H=\langle r^{16},s\rangle$, $H=\langle r^{8},s\rangle$, $H=\langle r^{4},s\rangle$.

\begin{minipage}{1\textwidth}
    \centering
    \includegraphics[width=0.9\linewidth]{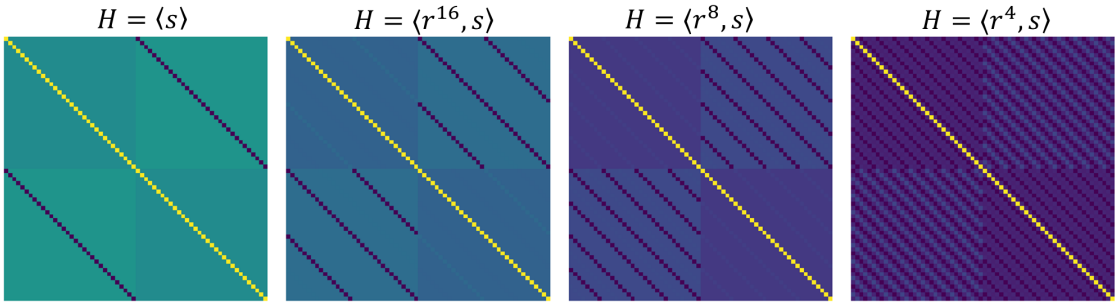}\
    \captionof{figure}{AGOPs learned by RFM with a Gaussian kernel trained on addition mod 32 on all data samples except for the fixed points of all the reflections $sr^k$ of the dihedral subgroups, in order from left to right: $H=\langle s\rangle$, $H=\langle r^{16},s\rangle$, $H=\langle r^{8},s\rangle$, $H=\langle r^{4},s\rangle$. RFM doesn't generalize to the withheld points in any of these settings.}
    \label{fig:dihedral_subgroups2}
\end{minipage}
\clearpage 

\section{Permutation representations} \label{sec:permutation_reps}

One can obtain the permutation representations of the symmetry group of each operation from the closed form of its elements and the one-hot encoded representation of the data. In our setup, following the data encodings used by \cite{gromov2023grokking, mallinar2025emergence}, group elements are represented by one-hot vectors. Let $G$ be a group of size $n$, and let each element $g \in G$ be encoded as a vector $e_g \in \mathbb{R}^n$, where $e_g$ is the one-hot encoding of $g$. Input data points consist of ordered pairs $(a, b) \in G \times G$, which are encoded as the concatenation $x = e_a \,\|\, e_b \in \mathbb{R}^{2n}.$ Group actions on $(a, b)$ correspond to linear transformations on $x \in \mathbb{R}^{2n}$, defined via block permutation matrices. For the symmetry group of the group operation (corresponding to addition and multiplication):
$$
r^x(a, b) = (a \star x,\; b \star x^{-1}), \quad s(a, b) = (b, a),
$$
we define permutation matrices $R_x, R_x^{-1} \in \mathbb{R}^{n \times n}$ corresponding to right multiplication by $x$ and $x^{-1}$, respectively. The reflection operator corresponds to swapping the two group elements. Then, the action of $r^x$ and $s$ on the one-hot input vector $x = e_a \,\|\, e_b$ is implemented by:
$$
\Pi(r^x) =
\begin{bmatrix}
R_x & 0 \\
0 & R_x^{-1}
\end{bmatrix}, \quad
\Pi(s) =
\begin{bmatrix}
0 & I_n \\
I_n & 0
\end{bmatrix}
$$
where $\Pi: G \to \mathrm{GL}(2n, \mathbb{R})$ denotes the permutation representation of the group action on data. We now consider the right-inverse operation $\tilde{f}(a,b)=a \star b^{-1}$ (corresponding to subtraction and division) with symmetry group denoted $\tilde{G}$ and defined by:
$$
\tilde{r}^x(a, b) = (a \star x,\; b \star x), \quad
\tilde{s}(a, b) = (b^{-1}, a^{-1}).
$$
This induces a modified linear representation $\tilde{\Pi}: \tilde{G} \to \mathrm{GL}(2n, \mathbb{R})$. Let $R_x \in \mathbb{R}^{n \times n}$ be as before, and define $\mathcal{I} \in \mathbb{R}^{n \times n}$ as the \textit{inversion matrix}, i.e., $I_{ij} = 1$ if $i = j^{-1}$, and zero otherwise. Then the representation matrices become:
$$
\tilde{\Pi}(\tilde{r}^x) =
\begin{bmatrix}
R_x & 0 \\
0 & R_x
\end{bmatrix}, \quad
\tilde{\Pi}(\tilde{s}) =
\begin{bmatrix}
0 & \mathcal{I}  \\
\mathcal{I}  & 0
\end{bmatrix}
$$
We proceed to compare the feature matrices learned by RFM when trained on all samples except for the fixed points under a reflection $sr^k$ with the permutation representation of such reflections for addition, subtraction, multiplication and division modulo 29 and all possible reflections.

\begin{minipage}{1\textwidth}
    \centering
    \includegraphics[width=1\linewidth]{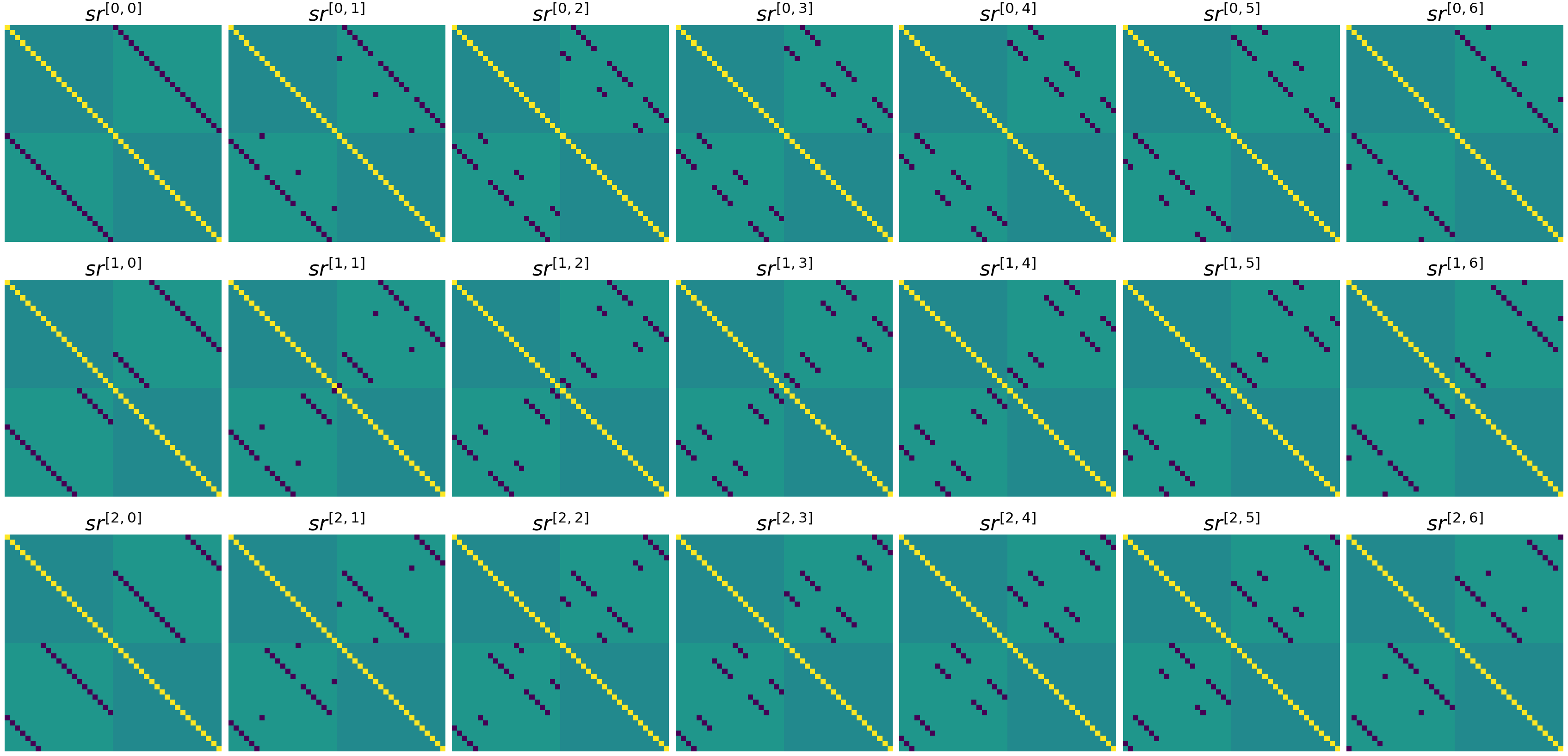}\
    \captionof{figure}{Feature matrices learned by RFM with a Gaussian kernel trained to solve addition for the Abelian group $\mathcal{A}\cong \mathbb{Z}_3\times\mathbb{Z}_7$ for 60 iterations on all data samples except the fixed points for a reflection $sr^x$ with $x\in\mathcal{A}$.}
    \label{fig:reflections_abelian}
\end{minipage}

\begin{minipage}{1\textwidth}
    \centering
    \includegraphics[width=0.49\linewidth]{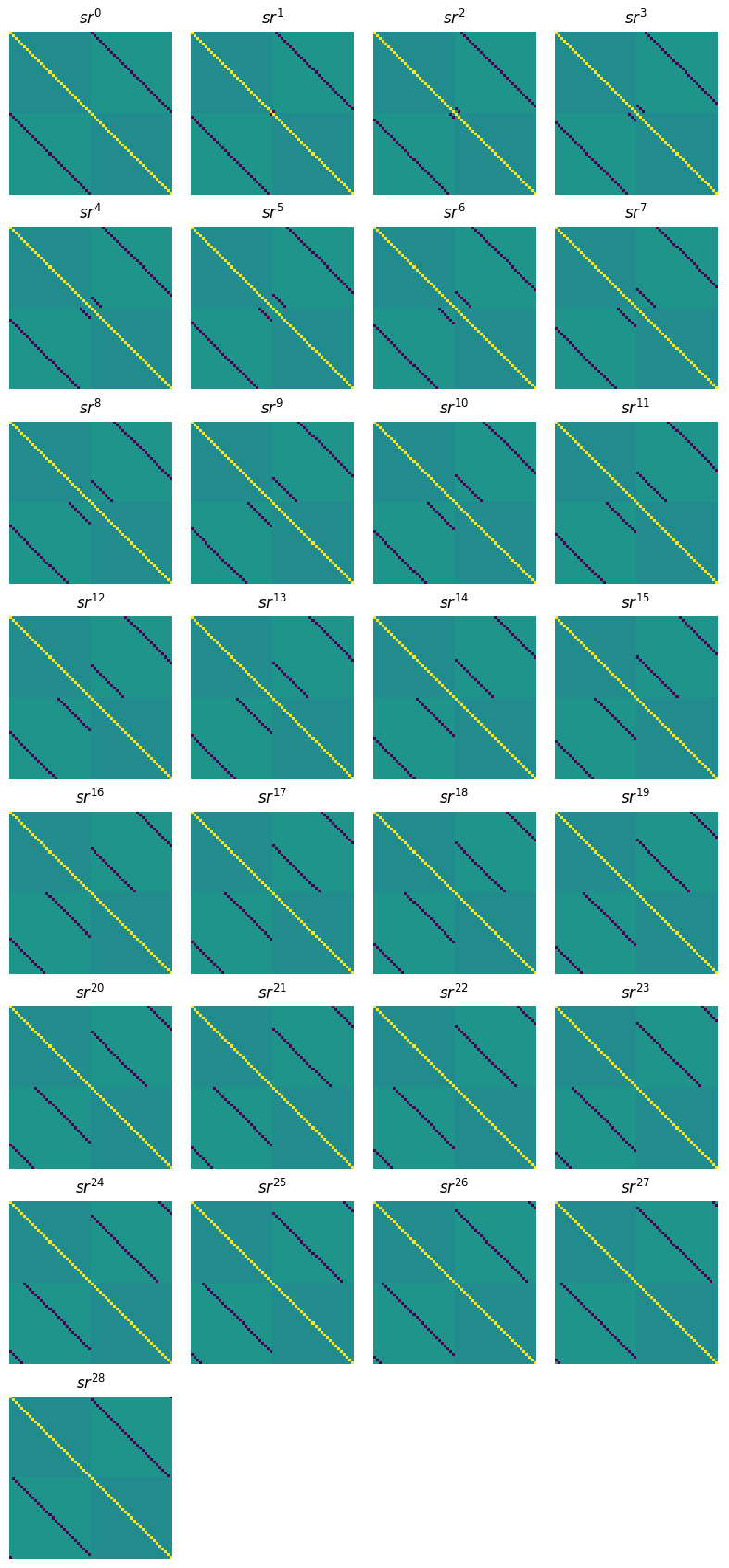}\
    \includegraphics[width=0.49\linewidth]{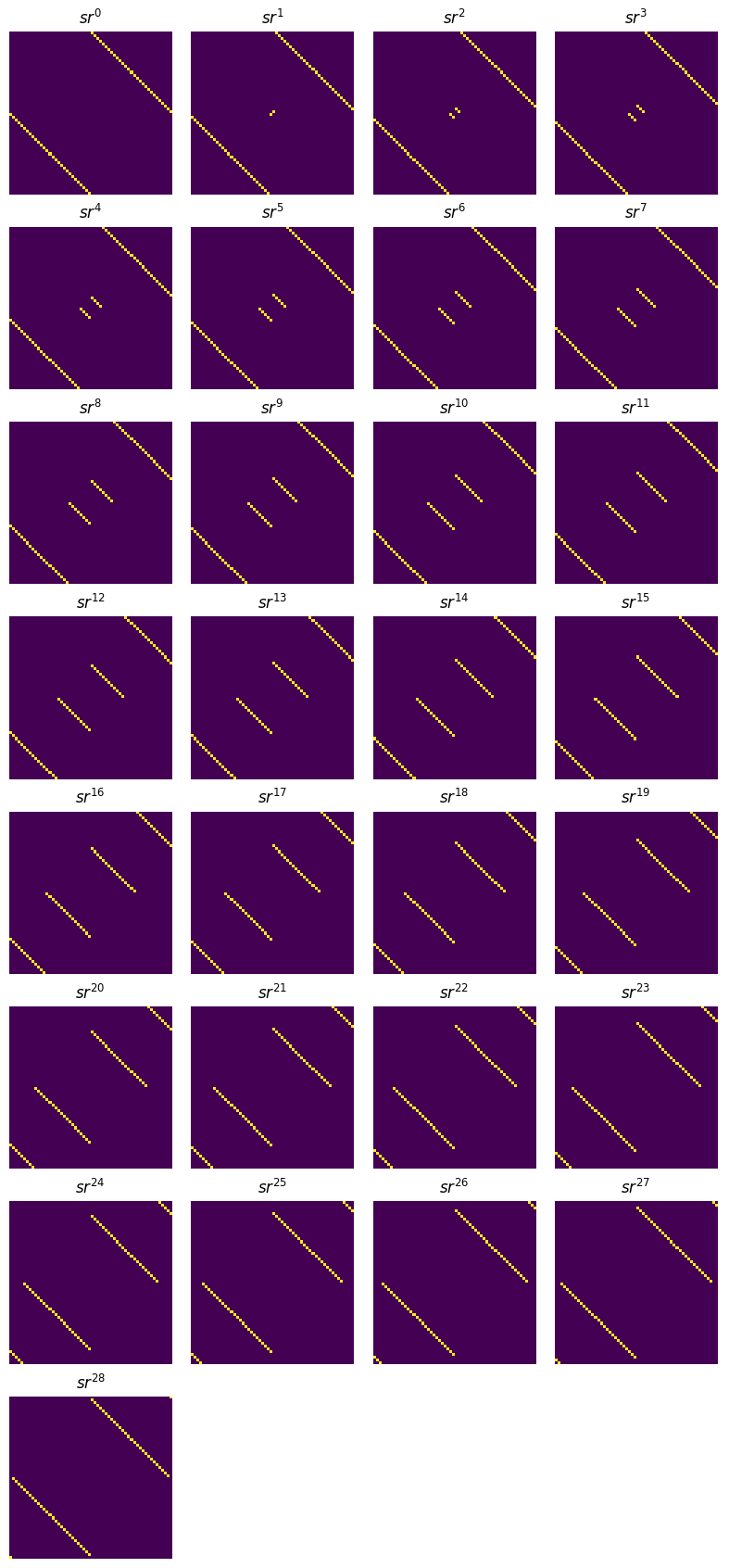}
    \captionof{figure}{On the left, feature matrices learned by RFM with a Gaussian kernel trained to solve addition modulo $29$ for 60 iterations on all data samples except the fixed points for a reflection $sr^k$. On the right, matrix representations of the reflections $sr^k$ according to representation theory.}
    \label{fig:reflections_add_29}
\end{minipage}

\begin{minipage}{1\textwidth}
    \centering
    \includegraphics[width=0.49\linewidth]{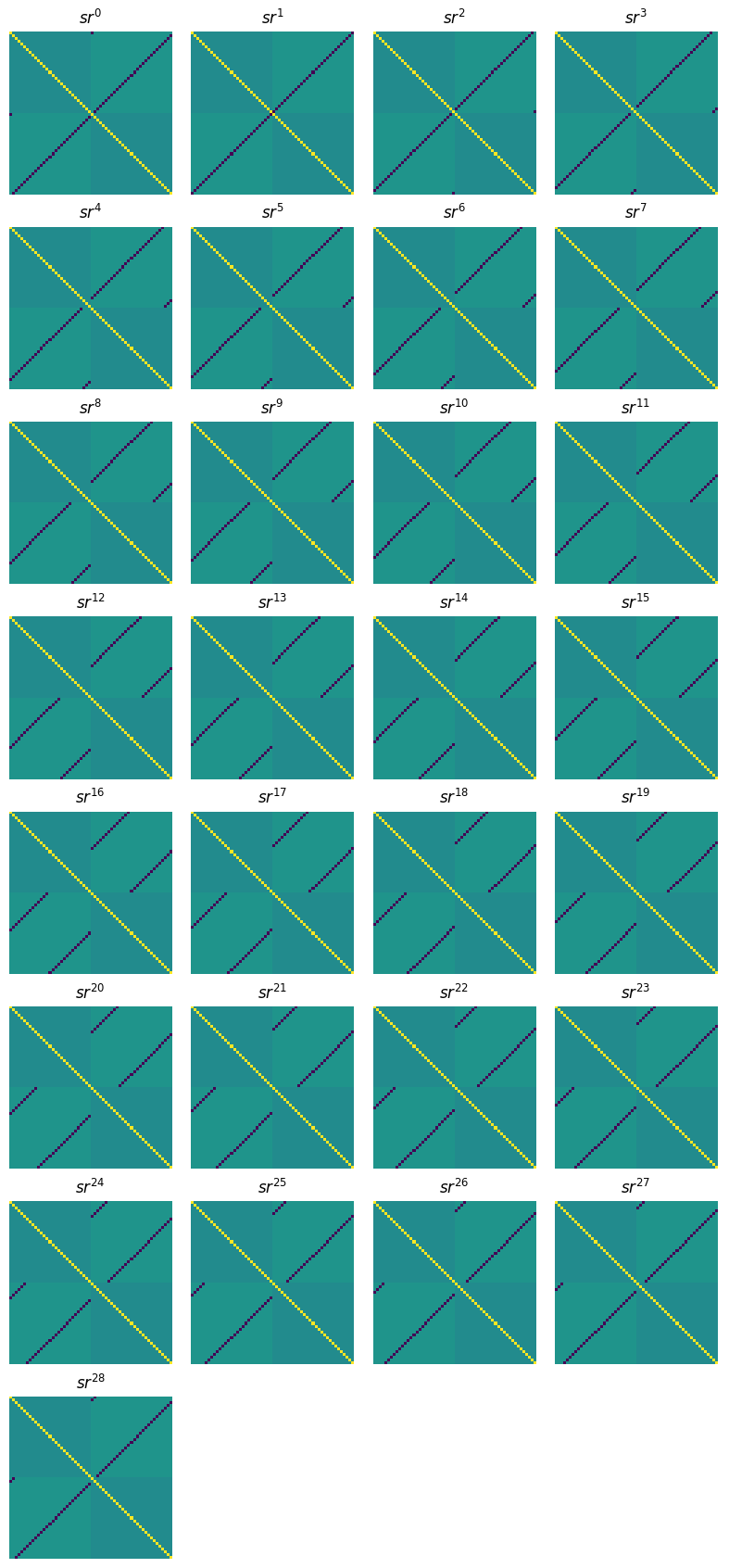}\
    \includegraphics[width=0.49\linewidth]{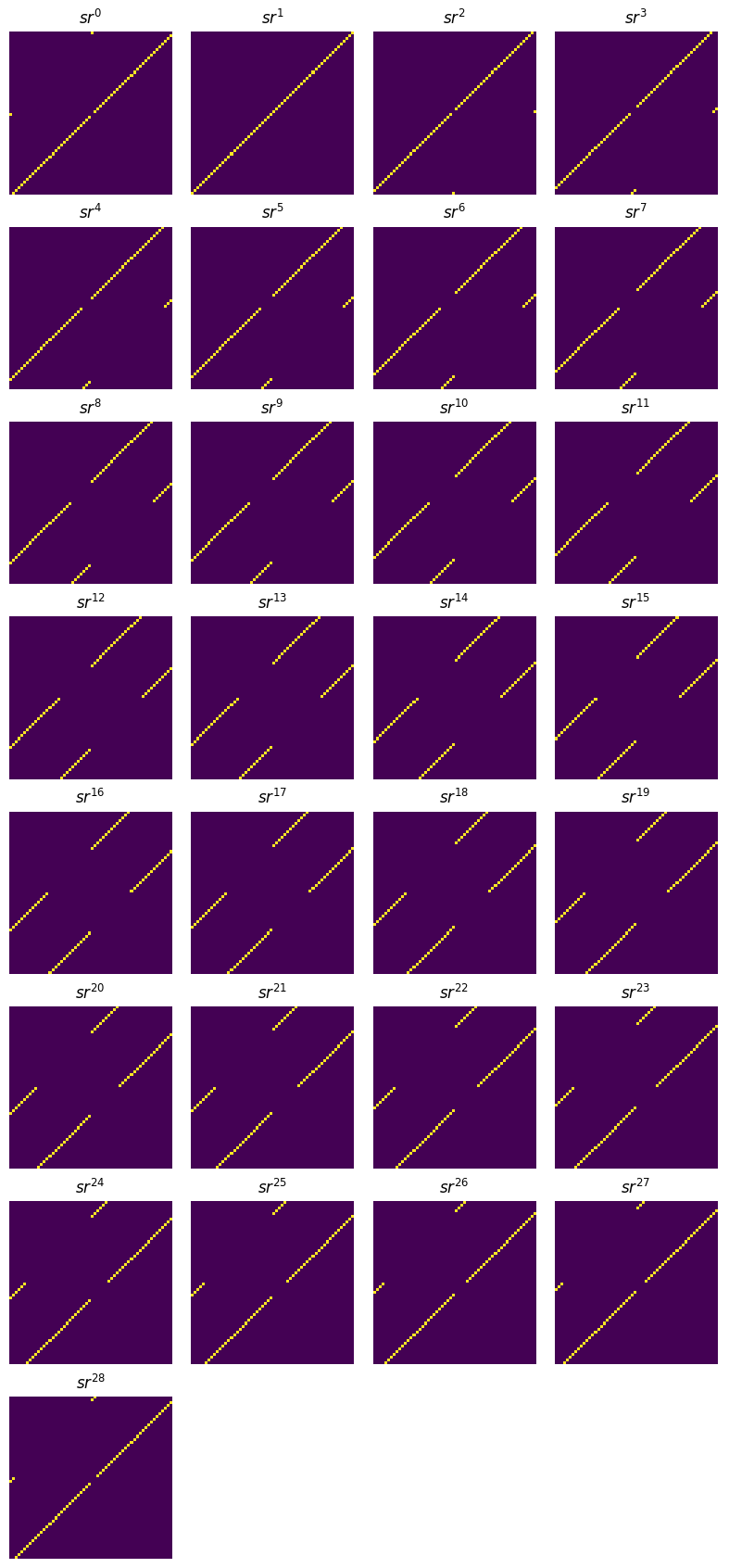}
    \captionof{figure}{On the left, feature matrices learned by RFM with a Gaussian kernel trained to solve subtraction modulo $29$ for 60 iterations on all data samples except the fixed points for a reflection $sr^k$. On the right, matrix representations of the reflections $sr^k$ according to representation theory.}
    \label{fig:reflections_sub_29}
\end{minipage}

\begin{minipage}{1\textwidth}
    \centering
    \includegraphics[width=0.49\linewidth]{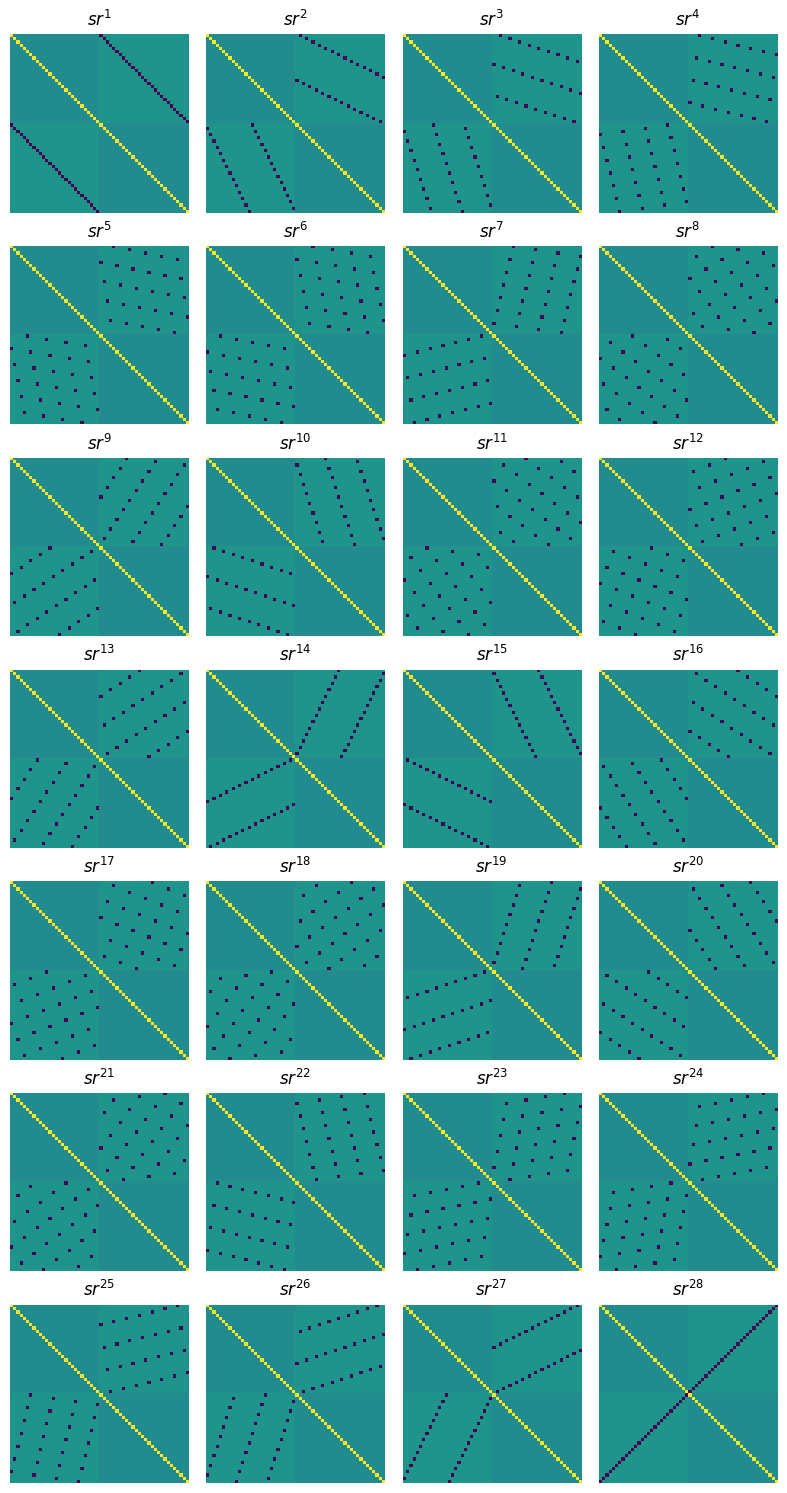}\
    \includegraphics[width=0.49\linewidth]{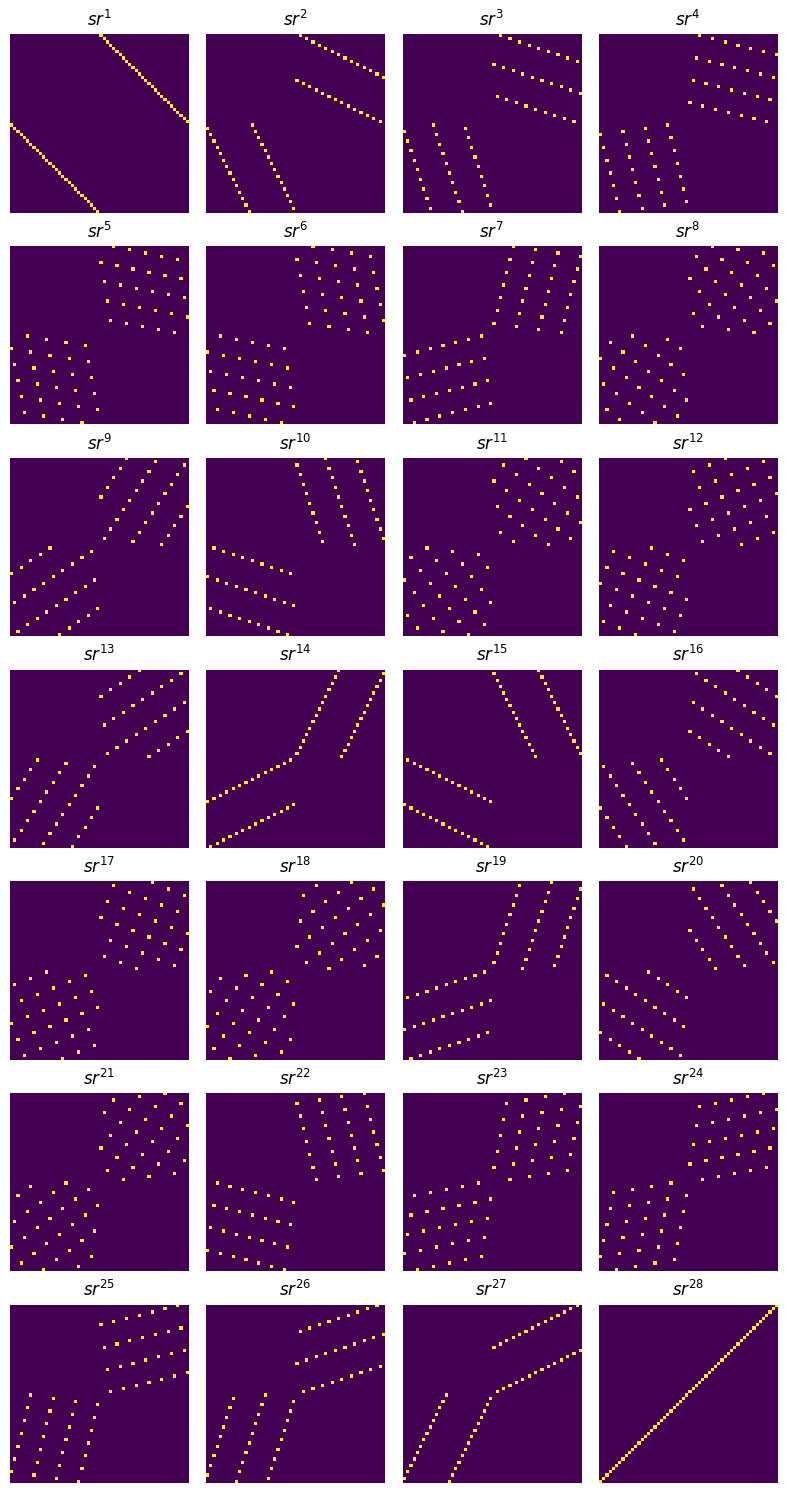}
    \captionof{figure}{On the left, feature matrices learned by RFM with a Gaussian kernel trained to solve multiplication modulo $29$ for 60 iterations on all data samples except the fixed points for a reflection $sr^k$. On the right, matrix representations of the reflections $sr^k$ according to representation theory.}
    \label{fig:reflections_multi_29}
\end{minipage}

\begin{minipage}{1\textwidth}
    \centering
    \includegraphics[width=0.49\linewidth]{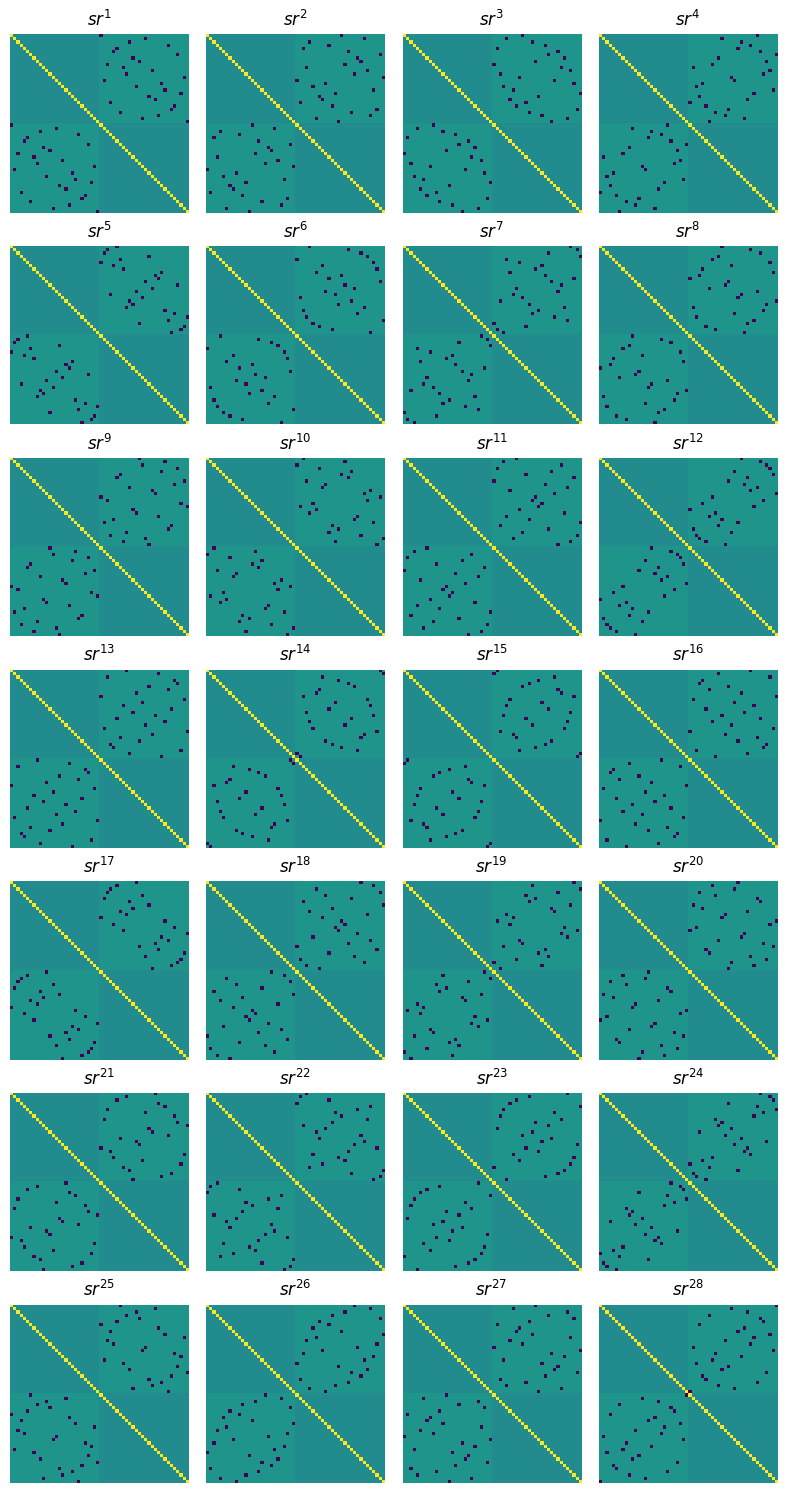}\
    \includegraphics[width=0.49\linewidth]{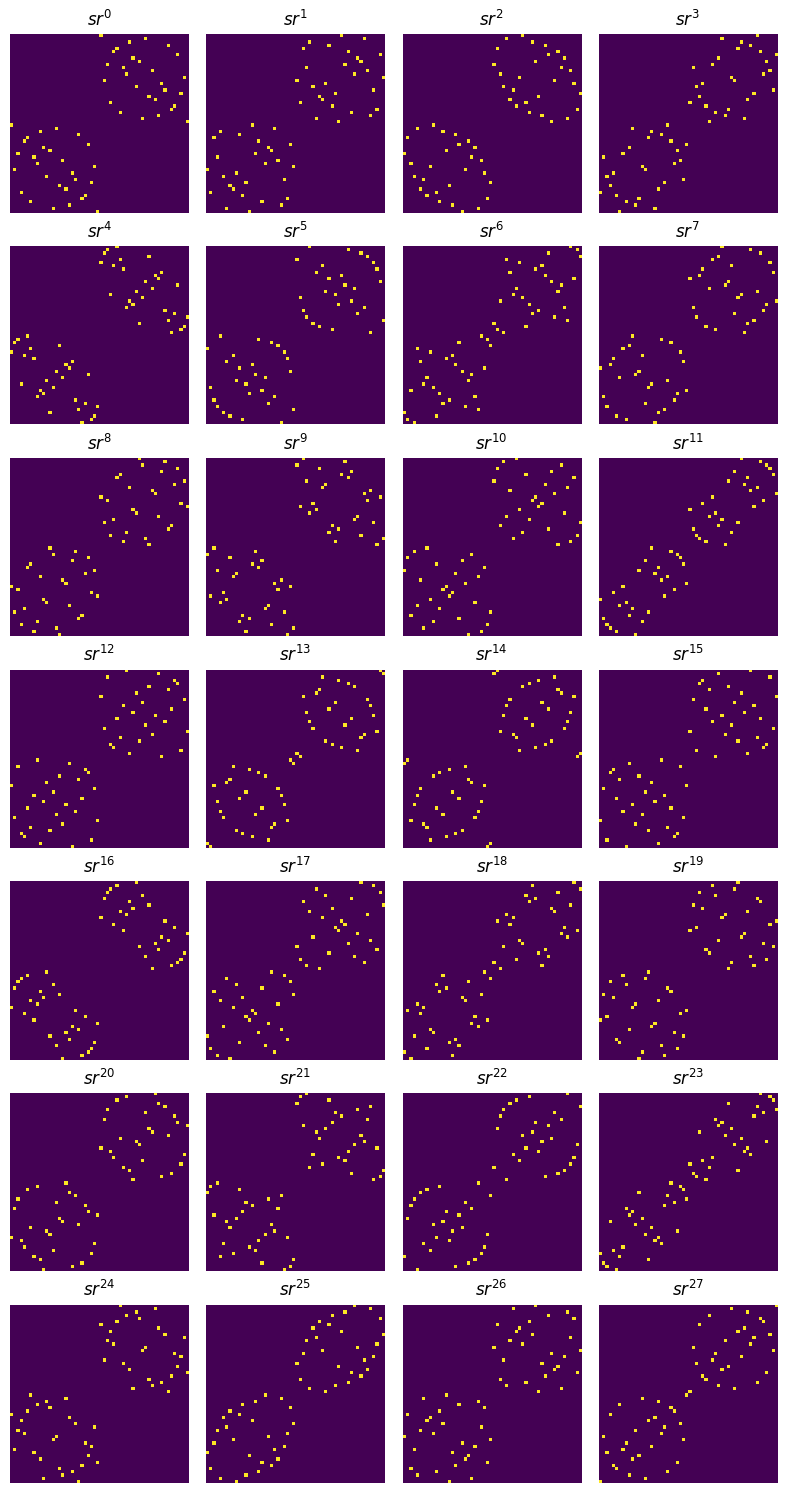}
    \captionof{figure}{On the left, feature matrices learned by RFM with a Gaussian kernel trained to solve division modulo $29$ for 60 iterations on all data samples except the fixed points for a reflection $sr^k$. On the right, matrix representations of the reflections $sr^k$ according to representation theory.}
    \label{fig:reflections_div_29}
\end{minipage}

\section{Generalization experiments}\label{sec:generalization_experiments}

In this appendix we generalize the results from Section \ref{features_generalization}. In Figure \ref{fig:rfm_predictions} we plot the correct and incorrect predictions of RFM for addition modulo 29 in the partitions described in Section \ref{sec:breaking}. In (A), we train RFM with a Gaussian kernel with $M_0=I_{2p}$ on all data samples except for the fixed points under $s$, the learned feature matrix aligns with the structure of the subgroup $H=\{\text{id},s\}$. The fixed points lay in the \textit{reflection axis} of $s$ (marked by the black line), and since they aren't in the train set, are incorrectly classified. In (B), we instead remove the fixed points under $sr^{10}$, but choose $M_0$ to encode the subgroup $H=\{\text{id},s\}$ (we set $M_0$ to be the learned feature matrix from the setting in A). Since the model \textit{learns} $s$ (Finding \ref{claim:generalization}), the reflection axis is still the same, so the test samples can be correctly classified. In (C), we repeat the setting from B, but removing the pairs under $s$ of some of the test samples. The test samples whose pairs have been removed from the train set are incorrectly classified, consistent with Finding \ref{claim:generalization}, since they are no longer in $\text{Orbit}_H(X_{tr})$. In (D), we remove all the pairs, which makes all the points in the test set be incorrectly classified.

\begin{minipage}{1\textwidth}
    \centering
    \includegraphics[width=0.8\linewidth]{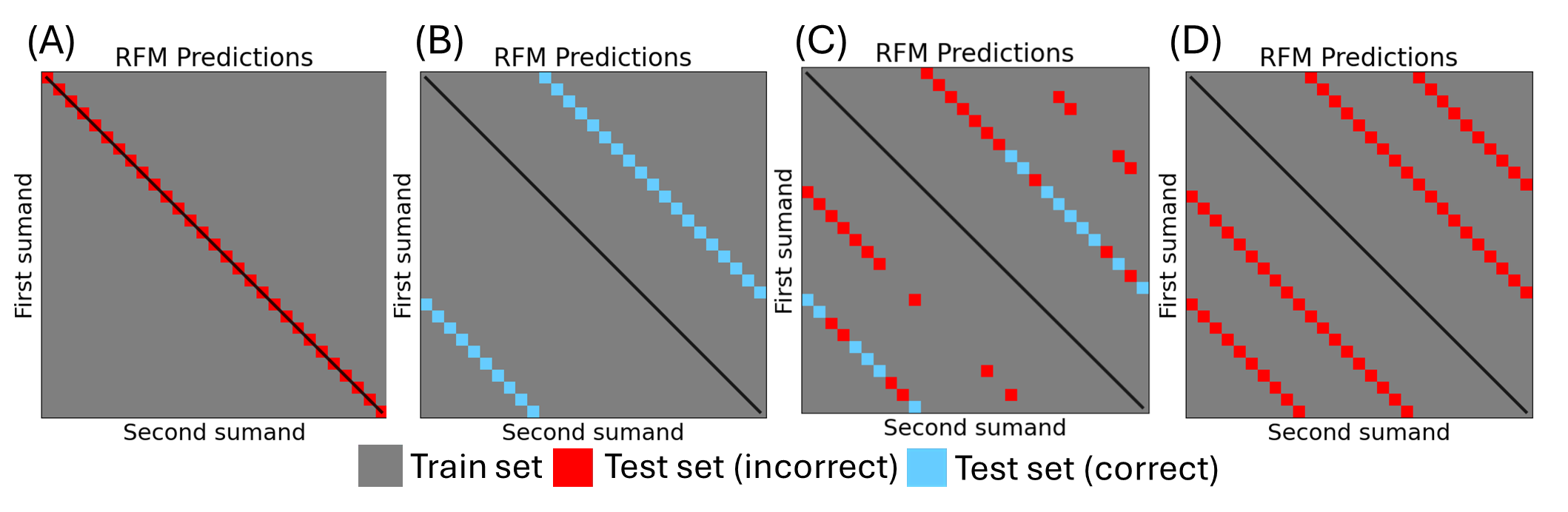}\
    \captionof{figure}{We train RFM with a Gaussian kernel on addition modulo $p=29$. (A) We train on all data samples except the fixed points under $s$ with $M_0=I_{2p}$. The fixed points aren't correctly classified. (B) We train on all data samples except the fixed points under $sr^{10}$ with $M_0$ encoding $H=\{\text{id},s\}$. (C) We repeat the setting from B, but we remove some of the pairs under $s$, the reflection encoded in $M_0$, of the points in the test set from the train set, which makes some samples be incorrectly classified. (D) We continue the pattern, this time removing all the pairs under $s$ of the points in the test set. All the test samples are incorrectly classified.}
    \label{fig:rfm_predictions}
\end{minipage}

We empirically verify Finding \ref{claim:generalization} for all possible reflections for addition modulo 29 (Figure \ref{fig:precisionrecall_add}) and subtraction modulo 29 (Figure \ref{fig:precisionrecall_sub}). We train on a random 50\% of the data, setting $M_0$ to encode $H=\{\text{id},sr^k\}$ for all possible $sr^k$. The figures show the Cayley tables for addition and subtraction mod 29 where the random training set elements are highlighted in grey, the incorrectly classified elements in the test set are higlighted in red, and the correctly classified test samples are highlighted in blue. We use a different color scheme from Figure \ref{fig:orbitpredictions} to improve readability. The structure of the matrices in Figures \ref{fig:precisionrecall_add} and \ref{fig:precisionrecall_sub} algins with Finding \ref{claim:generalization}. First, the samples in the corresponding reflection axis (highlighted by the black line) are either in the train set or incorrectly classified. Second, the incorrect predictions are symmetric with respect to the reflection axis, and if $(a,b)$ is a correct prediction, $(b-k,a-k)$ (it's reflection with respect to the axis) is in the train set. Repeating the words of Finding \ref{claim:generalization}, the test samples are only correctly classified if they are in $\text{Orbit}_{H}(X_{tr})$ for the chosen $H=\{\text{id},sr^k\}$ (marked in the plots by the black lines).

\begin{minipage}{1\textwidth}
    \centering
    \includegraphics[width=0.9\linewidth]{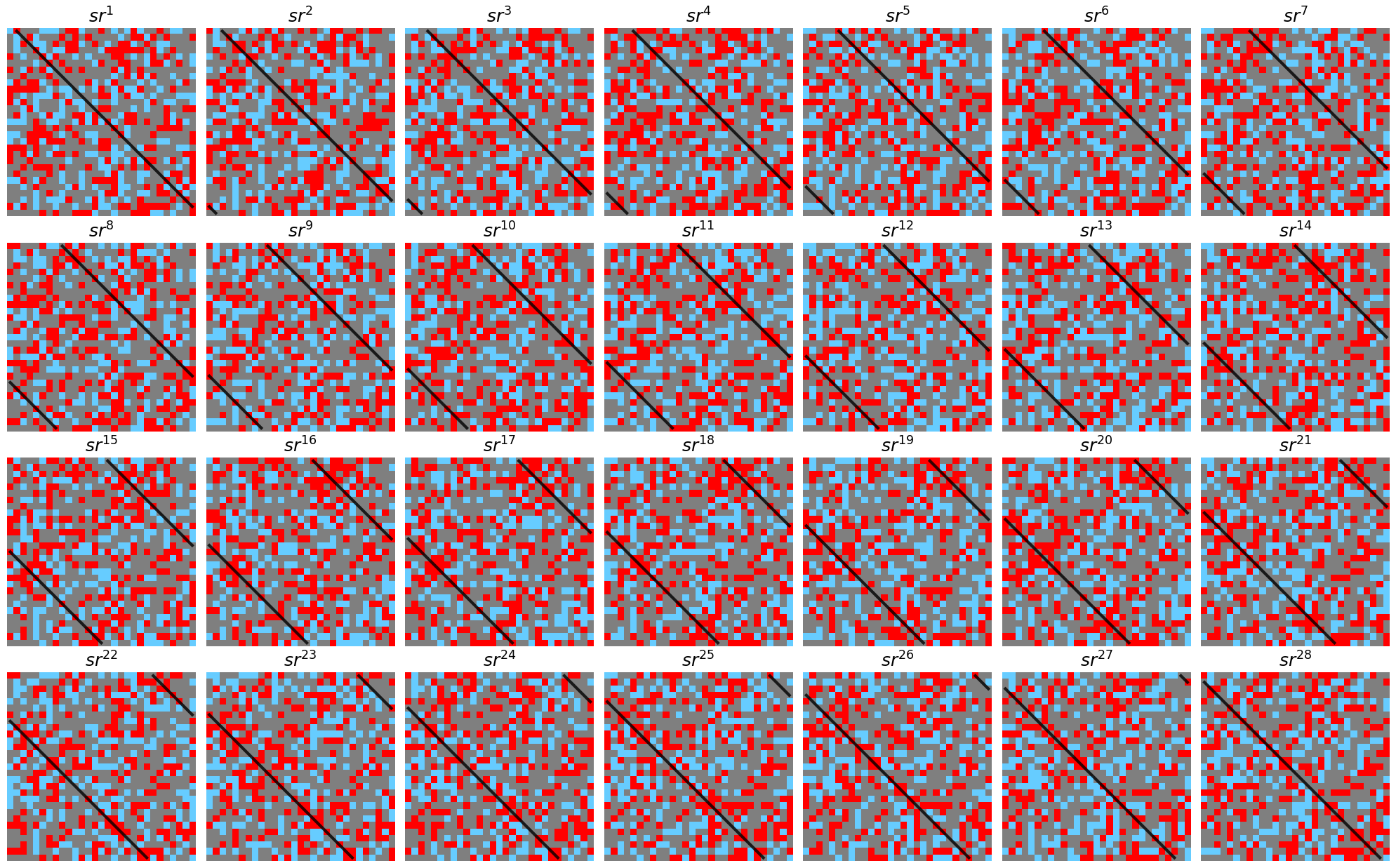}\
    \captionof{figure}{We train RFM with a Gaussian kernel on \textbf{addition }modulo $p=29$ on 50\% of the data with $M_0$ encoding the subgroup $H=\{\text{id},sr^{k}\}$. The reflection axis is marked by the black line. The square in the position $(i,j)$ encodes the sample $(i,j)$. Training samples are highlighted in grey, wrong predictions are highlighted in red, and correct predictions are highlighted in blue. There is a 100\% precision and 100\% recall correspondence with the correct predictions and the orbit of the train set with respect to $H$, $\text{Orbit}_H(X_{tr})$, which can be verified from the structure of the matrices with respect to the reflection axis.}
    \label{fig:precisionrecall_add}
\end{minipage}

\begin{minipage}{1\textwidth}
    \centering
    \includegraphics[width=0.9\linewidth]{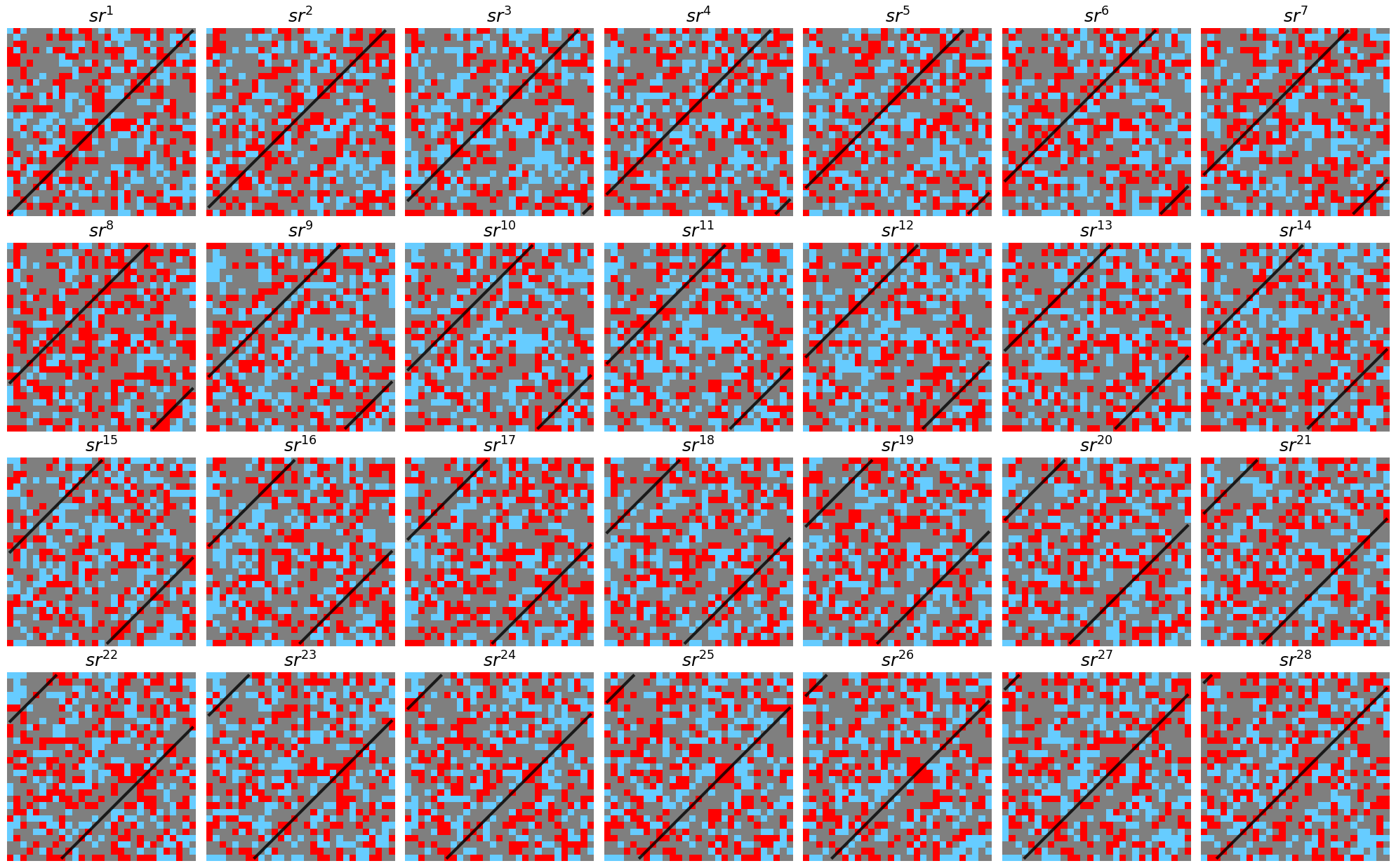}\
    \captionof{figure}{We train RFM with a Gaussian kernel on \textbf{subtraction} modulo $p=29$ on 50\% of the data with $M_0$ encoding the subgroup $H=\{\text{id},sr^{k}\}$. The reflection axis is marked by the black line. The square in the position $(i,j)$ encodes the sample $(i,j)$. Training samples are highlighted in grey, wrong predictions are highlighted in red, and correct predictions are highlighted in blue.}
    \label{fig:precisionrecall_sub}
\end{minipage}

\end{document}